\documentclass[lettersize,journal]{IEEEtran}
\usepackage{amsmath,amssymb,amsthm,graphicx,subcaption}
\usepackage{mathrsfs,algorithm,algpseudocode}
\usepackage{booktabs,multirow,bm,diagbox}
\usepackage{cite}
\usepackage{xcolor}
\usepackage{mathtools}
\usepackage{framed}
\graphicspath{{figures/}}

\begin{document}

% the copyright notice 
\onecolumn % switch to one column
\pagestyle{empty} % don't print page number
%\large\bfseries  % font for your notice
\begin{framed}
This work has been submitted to the IEEE for possible publication. Copyright may be transferred without notice, after which this version may no longer be accessible. 2025 IEEE. Personal use of this material is permitted. Permission from IEEE must be obtained for all other uses, in any current or future media, including reprinting/republishing this material for advertising or promotional purposes, creating new collective works, for resale or redistribution to servers or lists, or reuse of any copyrighted component of this work in other works.
\end{framed}
\twocolumn % switch to 2 columns
\setcounter{page}{1} % page number for start of manuscript
\pagestyle{headings} % per your code

% the environment of theorem
\newtheorem{theorem}{Theorem}
\newtheorem{exmp}{Example}
\newtheorem{lemma}{Lemma}
\newtheorem{definition}[exmp]{Definition}
% the environment of algorithm
\renewcommand{\algorithmicrequire}{\textbf{Input:}}
\renewcommand{\algorithmicensure}{\textbf{Output:}}
% table
\newcommand{\tabincell}[2]{\begin{tabular}{@{}#1@{}}#2\end{tabular}}

\title{Emulating Full Participation: An Effective and Fair Client Selection Strategy for Federated Learning}

\author{Qingming Li, Juzheng Miao, Puning Zhao, Li Zhou, H. Vicky Zhao, Shouling Ji, Bowen Zhou, Furui Liu
\thanks{This work was supported by National Science and Technology Major Project (2023ZD0121401). \emph{(Corresponding author: Furui Liu)}}
\thanks{Qingming Li, Shouling Ji are with the College of Computer Science and Technology at Zhejiang University, Hangzhou, Zhejiang, 310027, China. E-mail: \{liqm, sji\}@zju.edu.cn.}
\thanks{Juzheng Miao is with the Department of Computer Science and Engineering, The Chinese University of Hong Kong, Hong Kong, China. E-mail: jzmiao22@cse.cuhk.edu.hk.}
\thanks{Li Zhou, Furui Liu are with Zhejiang Lab, Hangzhou, Zhejiang, 311000, China. E-mail: \{pnzhao, zhou.li, liufurui\}@zhejianglab.com.}
\thanks{Puning Zhao is with School of Cyber Science and Technology, Shenzhen Campus of Sun Yat-sen University, Shenzhen 518107, China. E-mail: zhaopn@mail.sysu.edu.cn.}
\thanks{H. Vicky Zhao is with the Department of Automation, Beijing National Research Center for Information Science and Technology, Tsinghua University, Beijing 100084 P. R. China (email: vzhao@tsinghua.edu.cn).}
\thanks{Bowen Zhou is with the Department of Electronic Engineering, Tsinghua University, Beijing, 100084, China. E-mail: zhoubowen@tsinhua.edu.cn.}
% \thanks{H. Vicky Zhao is with the Department of Automation, Beijing National Research Center for Information Science and Technology, Tsinghua University, Beijing 100084 P. R. China (email: vzhao@tsinghua.edu.cn).}
}

\maketitle

\begin{abstract}
In federated learning, client selection is a critical problem that significantly impacts both model performance and fairness. Prior studies typically treat these two objectives separately, or balance them using simple weighting schemes. However, we observe that commonly used metrics for model performance and fairness often conflict with each other, and a straightforward weighted combination is insufficient to capture their complex interactions. To address this, we first propose two guiding principles that directly tackle the inherent conflict between the two metrics while reinforcing each other. Based on these principles, we formulate the client selection problem as a long-term optimization task, leveraging the Lyapunov function and the submodular nature of the problem to solve it effectively. Experiments show that the proposed method improves both model performance and fairness, guiding the system to converge comparably to full client participation. This improvement can be attributed to the fact that both model performance and fairness benefit from the diversity of the selected clients' data distributions. Our approach adaptively enhances this diversity by selecting clients based on their data distributions, thereby improving both model performance and fairness.
\end{abstract}

\begin{IEEEkeywords}
Federated Learning, Client Selection, Coreset Selection, Individual Fairness, Lyapunov Function
\end{IEEEkeywords}

\section{Introduction}
\label{sec:intro}

Federated learning (FL) facilitates collaborative model training without the necessity of sharing local data~\cite{mcmahan2017communication,kairouz2021advances} and is widely used in various domains~\cite{tan2022federated,yin2022fgc}. In FL, model parameters or gradient updates are frequently exchanged between the server and clients, which leads to substantial communication overhead. To address the challenge of limited bandwidth, a common approach is to select a subset of clients for local training~\cite{fu2023client,Bian2024accelerating}.  Therefore, a critical challenge is how to select proper and representative clients to participate, which is known as the client selection problem.
%However,  particularly for heterogeneous scenarios. 
%
%For example, in Fig.~\ref{fig:intro} (a), clients collaboratively train an image classification model, where each client possesses images of only one type of animal. In such cases, 

%In homogeneous scenarios where clients hold independent and identically distributed (i.i.d.) data, random selection is commonly employed. However, 
In heterogeneous scenarios—where clients hold training data with diverse distributions—the client selection problem becomes particularly critical. Specifically, client selection impacts the FL system in two key aspects. First, client selection influences \textbf{model performance}. Different selection strategies lead to distinct optimization trajectories, and an improper client selection may cause the optimization process to deviate significantly from the optimal path. Second, client selection affects \textbf{fairness}. The global model tends to perform better for frequently selected clients, as their data is better optimized, while producing biased predictions for less frequently selected clients. Since every client participates in the system with the expectation of obtaining accurate predictions for its own data, this selection imbalance may drive underrepresented clients to leave the system. Therefore, it is essential to establish an effective and fair client selection strategy.

\textbf{Limitation of Prior Works.} Prior studies usually treat model performance and fairness as separate objectives and attempt to balance the two using simple weighting schemes~\cite{du2021fairness, huang2020efficiency, zhu2022online}. However, model performance and fairness metrics commonly in use are often conflicting, and a simple weighted combination is insufficient to capture their intricate interactions. Specifically, to achieve high model performance, existing client selection methods~\cite{goetz2019active,cho2022towards,tang2022fedcor} often prioritize clients with higher training losses, as these clients are considered more challenging to fit their local data. However, such loss-guided methods can lead to biased predictions for clients that are rarely selected. On the other hand, to enhance fairness, existing approaches typically focus on performance fairness~\cite{shi2023towards, li2021ditto,du2021fairness, huang2020efficiency, zhu2022online}, which requires the model to deliver similar performance or uniform selection probabilities across all clients. However, this uniform selection approach overlooks the heterogeneous nature of clients' data distributions, thereby sacrificing model performance. Therefore, a critical challenge lies in reconciling the conflicting goals of model performance and fairness, particularly in heterogeneous scenarios.

\begin{figure}[tbp]
    \centering
    \includegraphics[width=8cm]{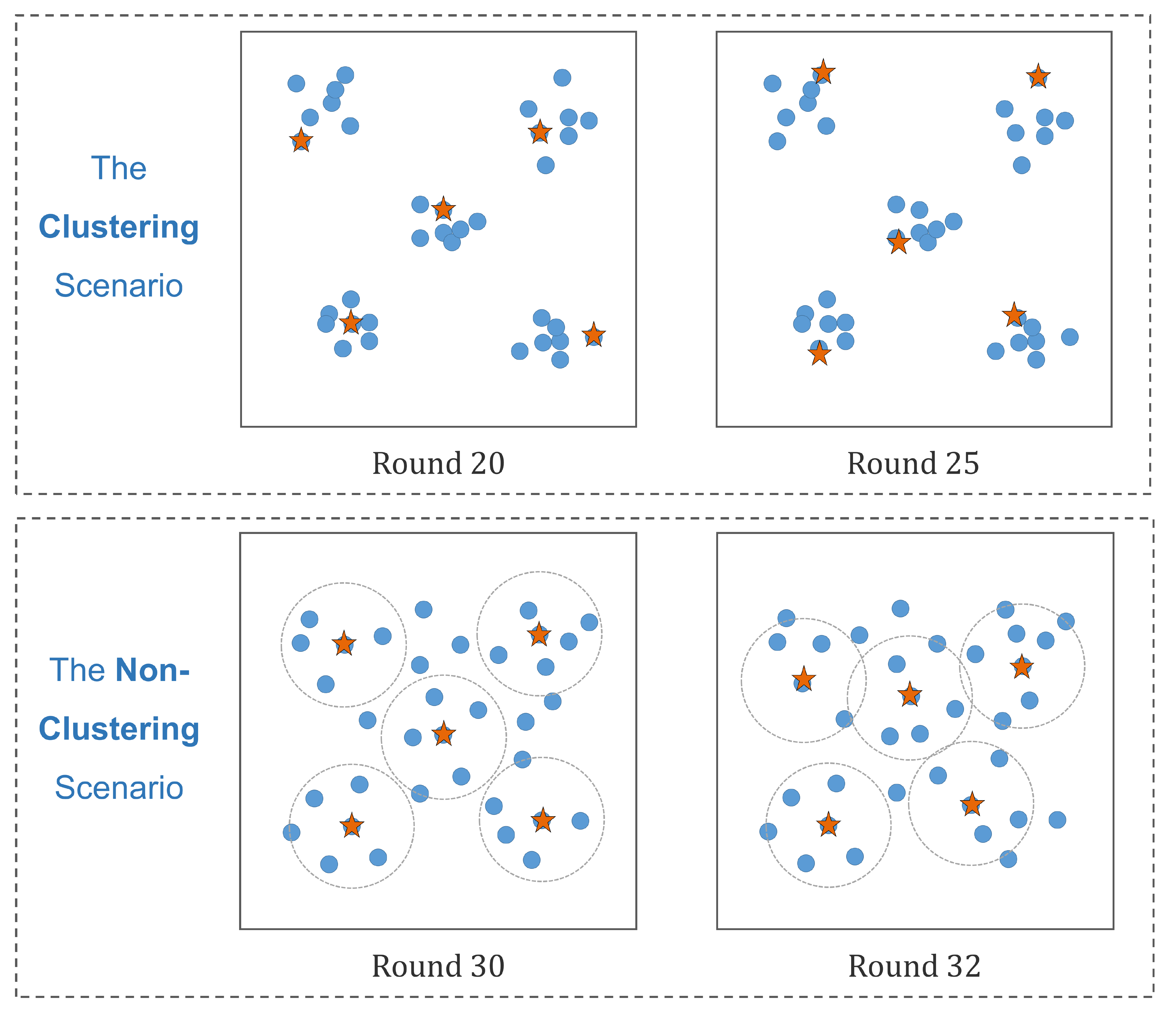}
    % \vspace{-0.1cm}
    \caption{Visualization of the client selection results. Clients are represented by blue dots, and selected clients are marked with orange stars. In the non-clustering scenario, clients arranged in a circle are effectively represented by the selected client positioned at the center.}
    \label{fig:intro}
\end{figure}

To design an effective and fair client selection strategy, we propose two guiding principles. These principles address the inherent conflict between the two metrics by capturing their interactions. First, from the perspective of model performance, we propose \textbf{Principle I}: \emph{The data distribution generated by the selected subset of clients should closely resemble the data distribution of the full client participation}. Full client participation is chosen as the standard because models trained with all clients generally yield robust results, serving as a comprehensive benchmark—except in certain extreme cases where some clients possess highly noisy or corrupted data that can skew the model's performance. Importantly, for clients that are less frequently selected in the loss-guided methods~\cite{goetz2019active,cho2022towards,tang2022fedcor}, this strategy ensures they are not excluded from selection by approximating the distribution of full client participation. 

Furthermore, from the perspective of fairness, we introduce \textbf{Principle II}: \emph{clients with similar data distributions should have similar frequencies of being selected.} This principle aligns with the concept of Individual Fairness (IF)~\cite{dwork2012fairness}, which posits that instances with similar features in a dataset should yield similar predictions or outcomes. We refer to this principle as \emph{the individual fairness constraint}. Unlike the uniform sampling strategies used in prior works~\cite{huang2020efficiency, zhu2022online, shi2023towards}, this constraint specifically focuses on clients with similar data distributions. It is better suited to heterogeneous data scenarios, as it avoids sacrificing model performance by preventing the uniform selection of clients who contribute minimally to model improvement.

The two principles are not independent and mutually reinforce each other. Consider an extreme case where there are two clusters of clients, and the data distributions are similar within each cluster. If only Principle II is applied, it may result in the system selecting clients exclusively from one cluster while neglecting the other. This phenomenon was observed in our experimental results, whereas Principle I helps mitigate this issue. On the other hand, if only Principle I is applied, the system may end up selecting the same client from each cluster repeatedly, whereas Principle II prevents this bias.

In this study, we integrate the two principles and propose an effective and fair client selection strategy for federated learning, called \emph{LongFed}. We begin by providing a mathematical formulation for both Principle I and Principle II. Then, we model the client selection problem as a long-term optimization function, introducing a tradeoff factor to balance the two principles. To solve this optimization, we simplify it using the Lyapunov optimization from control theory, and propose a fast greedy algorithm based on the submodular nature of the problem. We also theoretically analyze the convergence of the proposed strategy, demonstrating that it converges at a rate of $\mathcal{O}(1/t)$, which is the same as loss-guided selection methods~\cite{goetz2019active,cho2022towards,tang2022fedcor,balakrishnan2022diverse}.

We evaluate the proposed strategy through extensive experiments, with results demonstrating that our method enhances both model performance and fairness. Regarding model performance, the strategy effectively guides the system to converge along a trajectory similar to that of full client participation, outperforming prior methods in most cases. Additionally, the strategy achieves strong fairness, evidenced by a low standard deviation in the selection frequencies of clients with similar data distributions. This improvement can be attributed to the fact that both model performance and fairness benefit from the diversity of the selected clients' data distributions, which is promoted by our two proposed principles. Unlike existing fair federated learning approaches~\cite{huang2020efficiency, zhu2022online, shi2023towards} that treat all clients equally, our method adaptively enhances this diversity based on clients' data distributions, making it better suited for heterogeneous environments. Additionally, our method introduces only a marginal time increase (less than 0.4 ms) compared to existing approaches.

We also provide visualization results to illustrate the effectiveness of our method. An example is shown in Fig.~\ref{fig:intro}. When clients exhibit clear clustering patterns, the proposed strategy selects one client from each cluster and chooses different clients across multiple rounds. This selection process aligns with the clustered federated learning~\cite{caldarola2021cluster, ghosh2020efficient}. More importantly, in scenarios where the clustering pattern is not evident—which is more common in practical cases—the proposed strategy selects clients that cover a majority of client population. This diversity not only enhances model performance but also improves fairness.

Our contributions are as follows. 
\begin{itemize}
    \item We identify the inherent conflict between model performance and fairness in the client selection problem and propose a strategy that leads to improvements in both. 
    \item We address a previously underexplored issue: the frequency of selecting clients with similar data distributions. We introduce an individual fairness criterion to mathematically formulate and effectively resolve this issue.
    \item We provide an extensive theoretical analysis of the convergence ability of the proposed strategy.
\end{itemize}

\section{Related Work}
\label{sec:review}
Existing literature relevant to our work can be broadly categorized into two groups: client selection methods and fairness issues in federated learning.

\subsection{Client Selection in Federated Learning}
\label{sec:client_selection_review}
In vanilla federated learning systems \cite{mcmahan2017communication}, the random selection strategy is commonly employed to choose clients. Recent works have proposed various improvements, including contribution-based \cite{shyn2021fedccea,mcmahan2017communication,liu2022gtg,wang2019measure,sun2023shapleyfl}, loss-based \cite{goetz2019active,cho2022towards,tang2022fedcor}, and cluster-based methods \cite{caldarola2021cluster,ghosh2020efficient}.

The first type focuses on evaluating client contributions, where clients with higher contributions are assigned higher selection probabilities. One commonly used metric is the local data size, and clients with larger datasets are considered to have higher contributions \cite{shyn2021fedccea,mcmahan2017communication}. Another approach involves employing the Shapley value \cite{shapley1953value} from game theory, which calculates the average marginal model improvement of each client over all possible coalitions \cite{liu2022gtg,wang2019measure,sun2023shapleyfl}. Clients that result in larger model performance improvements are regarded as making larger contributions and are assigned higher probabilities of selection. Additionally, some methods evaluate the similarity between the local model at each client and the aggregated global model at the server. Clients with higher similarities are considered to bring little improvement to the global model and are thus assigned lower selection probabilities in subsequent training epochs \cite{jiang2023fair,karimireddy2020scaffold}.

The second type involves selecting clients based on their training losses. These methods consider that clients with higher training losses may struggle to effectively fit their local data. As a result, they assign higher selection probabilities or weights to these clients \cite{goetz2019active,cho2022towards}. FedCor is a representative work in loss-based approaches that employs Gaussian processes to model the loss correlations between clients and selects clients with a substantial reduction in expected global loss \cite{tang2022fedcor}. Furthermore, recognizing that clients may contain similar and redundant information, a diverse strategy is proposed to choose a subset of clients that can best represent the full client set \cite{balakrishnan2022diverse}. 

The third approach involves a cluster-based strategy. Recognizing that clients may have different data distributions or have their own learning tasks, these clustering-based approaches divide the clients into several clusters. Clients within the same cluster exhibit similar data distributions, whereas those in different clusters may display significant variations in their data distributions. In aggregation, the server randomly selects a client from each cluster. In this cluster-based paradigm, FedCG leverages a graph neural network to capture gradient sharing across multiple clusters \cite{caldarola2021cluster}. Moreover, IFCA addresses the cluster identification problem by determining the cluster membership of each client and optimizes each of the cluster models in a federated learning framework \cite{ghosh2020efficient}.

\subsection{Fairness in Federated Learning}
There are three types of fairness in federated learning: collaborative fairness, group fairness, and performance fairness. Collaborative fairness emphasizes that clients who make larger contributions should be rewarded with correspondingly larger rewards, and the assessment of client contributions is a key challenge. Client contribution evaluation methods have been discussed in Section \ref{sec:client_selection_review} in the context of contribution-based client selection methods. 

Group fairness \cite{mehrabi2021survey}, also known as algorithmic fairness, emphasizes that model outputs should not unfairly discriminate against vulnerable or underrepresented groups, such as minorities, women, or the aged \cite{chen2023privacy}. FairFed \cite{ezzeldin2023fairfed} serves as one of the representative works of addressing this concern. In FairFed, both global fairness metrics at the server and local fairness metrics at each client are defined. These metrics assess disparities in opportunities among different groups. FairFed then dynamically adjusts aggregation weights at each round based on the discrepancies between global and local fairness metrics, which effectively mitigates the biases towards sensitive attributes \cite{ezzeldin2023fairfed}. 

Performance fairness refers that the model should produce similar performance across all clients, aligning most closely with our objectives. Various approaches have emerged to achieve performance fairness. For example, in \cite{du2021fairness}, fairness is defined by assigning uniform weight to each client and introduced as an additional constraint in the optimization function. Besides, \cite{li2021ditto} identifies a trade-off between model robustness and fairness, and proposes a personalized framework to inherently achieve both fairness and robustness benefits. Moreover, several works study the fairness issue by considering that the probability of being selected is similar for all clients \cite{huang2020efficiency,zhu2022online,shi2023fairness}. However, as introduced in Section \ref{sec:intro}, the uniform selection method disregards the heterogeneous data distribution among clients. 

In summary, existing client selection methods only consider the model performance and do not address the fairness. Although some fair federated learning methods have been proposed, the uniform selection constraint adopted in their methods could disregard the heterogeneous data distribution among clients. Therefore, it remains a critical problem for client selection to simultaneously address model performance and fairness. 

\section{The Proposed Optimization Function}
\label{sec:model}
In this section, we begin with an introduction to the federated learning system. Next, we provide a mathematical formulation for both Principle I and Principle II, and model the client selection problem as a long-term optimization function. Last, we apply Lyapunov optimization to simplify the optimization function.
 
\subsection{Preliminary of Federated Learning}
In our work, we consider that the federated learning system consists of a central server and a set of clients denoted as $\mathbb{N}=\{1,\cdots, N\}$.  
Each client $i\in \mathbb{N}$ has its own local dataset $\mathcal{D}_i$ with a size of $|\mathcal{D}_i|$. 
In the $t$-th round, the server selects a subset of clients, denoted as $\mathbb{S}^t$ with $|\mathbb{S}^t|=K<N$. Clients within the subset $\mathbb{S}^t$ receive the global model $\boldsymbol{w}^t$ from the server. They then compute local updates on their respective local datasets, transmitting the local gradients $\nabla f_j(\boldsymbol{w}^t)$ back to the server. The server aggregates these gradients and updates the model using
\begin{align}
\label{equ:aggregate}
\boldsymbol{w}^{t+1}=\boldsymbol{w}^t-\eta_t\sum_{j\in \mathbb{S}^t}\theta^t_j \nabla f_j(\boldsymbol{w}^t).
\end{align}
Here, $\eta_t$ represents the predefined learning rate, and $\theta^t_j$ denotes the weight assigned to client $j\in\mathbb{S}^t$ in the $t$-th round. The training process lasts for $T$ rounds until the global model achieves convergence. The subset $\mathbb{S}^t$ and the weights $\theta^t_j$ with $j\in\mathbb{S}^t$ are the variables to be determined. 

\subsection{The Optimization Function}
\label{sec:optimization}

\subsubsection{Formulation of Principle I}
\label{sec:diversity}
In our work, we employ the Principle I to guide the client selection in a single round. Specifically, we evaluate the estimation error between the aggregated gradient obtained from the client subset $\mathbb{S}^t$ and the aggregated gradient obtained from the full client set $\mathbb{N}$. A small estimation error indicates that the subset of selected clients can effectively represent the data distribution of the full client set. % and the optimization direction with the subset of clients is similar to that with full client participation. 
Mathematically, we formulate it as 
\begin{align}
\label{equ:diversity}
\text{D}(\mathbb{S}^t)= \min_{\theta_j^t>0}||\sum_{i\in \mathbb{N}} \nabla f_i(\boldsymbol{w}^t) - \sum_{j\in \mathbb{S}^t } \theta_j^t \nabla f_j(\boldsymbol{w}^t) ||_2^2.
\end{align}
Here, $\nabla f_i(\boldsymbol{w}^t)$ represents the gradient of the $i$-th client, $\sum_{i\in \mathbb{N}} \nabla f_i(\boldsymbol{w}^t)$ is the aggregated gradient on the full client set $\mathbb{N}$, and $\sum_{j\in \mathbb{S}^t } \theta_j^t \nabla f_j(\boldsymbol{w}^t)$ denotes the weighted sum of gradients on the client subset $\mathbb{S}^t$. 
Then, selecting a subset of clients that best approximates the data distribution of the full client set is equivalent to choosing a subset that minimizes the estimation error defined in Eq.~\eqref{equ:diversity}.

Notice that the formation in Eq. \eqref{equ:diversity} is similar to the concept of data coreset introduced in \cite{mirzasoleiman2020coresets}. The data coreset involves selecting a weighted small sample of training data to approximate the gradient of the whole training data set. It is important to highlight that data coreset is commonly employed to enhance training efficiency in centralized machine learning scenarios \cite{huang2021coresets,killamsetty2021glister}, while our focus is on optimizing client selection to reduce communication bandwidth in decentralized and federated scenarios. Moreover, data coreset is typically utilized for a one-time data selection process, while client selection in federated learning is a long-term process carried out across multiple rounds. The long-term optimization nature introduces new challenges, such as fairness issues and potential model biases introduced in Section \ref{sec:intro}.

%In our study, we extend this notion to define intra-round diversity within a subset of clients in federated learning. 

In Eq.~\eqref{equ:diversity}, when clients are not selected in the $t$-th round, their gradients $\nabla f_i(\boldsymbol{w}^t)$ become unknown. To address this challenge, following the analysis in \cite{mirzasoleiman2020coresets}, we assume a mapping $\xi^t:\mathbb{N}\rightarrow \mathbb{S}^t$ that assigns each client $i\in\mathbb{N}$ to a client $j\in\mathbb{S}^t$, i.e., $\xi^t(i)=j$, indicating that client $i$ can be approximately represented by client $j$ in the $t$-th round. For a client $j\in\mathbb{S}^t$, let $\mathbb{C}^t_j=\{i\in \mathbb{N}\,|\,\xi^t(i)=j\}$ be the set of clients that can be represented by client $j$. The value $\theta^t_j=|\mathbb{C}^t_j|$ is the number of such clients being represented, and is used as the weight of client $j$ in Eq. \eqref{equ:aggregate}. Based on the mapping $\xi^t$, we obtain the upper bound of the estimation error in  Eq. \eqref{equ:diversity}, as stated in Theorem \ref{theo:diversity_bound}. In the later section, we utilize $\text{DUB}(\mathbb{S}^t)$ when minimizing the estimation error $\text{D} (\mathbb{S}^t)$ is required. The proof of Theorem~\ref{theo:diversity_bound} is provided in supplementary file. 
\begin{theorem}
\label{theo:diversity_bound}
Define 
\begin{align}
\label{equ:simi}
\text{Dist}_{i,j}(t) = ||\nabla f_i(\boldsymbol{w}^t) - \nabla f_j(\boldsymbol{w}^t)||_2^2, 
\end{align}
and 
\begin{align}
\label{equ:diversity_bound}
    \text{DUB}(\mathbb{S}^t) &\triangleq \sum_{i=1}^{N} \min_{j\in \mathbb{S}^t} \text{Dist}_{i,j}(t),
\end{align}
then $\text{DUB}(\mathbb{S}^t)$ serves as an upper bound for $\text{D}(\mathbb{S}^t)$.
\end{theorem}

In our work, we use partial updates to compute $\text{Dist}_{i,j}(t) $. Specifically, when $t=0$, all clients are selected, and we compute $\text{Dist}_{i,j}(t) $ for each pair of clients. Then, for $t\geq 1$, the server updates $\text{Dist}_{i,j}(t)$ using the gradients only from the selected $K$ clients. That is,
\begin{align}
    \label{equ:simi_update}
    \text{Dist}_{i,j}(t) =  \begin{cases}
        & \Vert\nabla f_i(\boldsymbol{w}^{t}) - \nabla f_j(\boldsymbol{w}^{t})\Vert_2^2, \;  \text{if} \;i,j\in\mathbb{S}^{t},\\ 
        & \text{Dist}_{i,j}(t-1),  \; \text{otherwise}. 
    \end{cases}
\end{align}
Although using $\text{Dist}_{i,j}(t-1)$ to approximate $\text{Dist}_{i,j}(t)$ may introduce biases, experimental results demonstrate that there are minimal impacts on the model convergence. This is because there is limited gradient variations between successive communication rounds, and the proposed individual fairness constraint ensures that all clients have the opportunity of being selected and the corresponding $\text{Dist}_{i,j}(t)$ can be updated. 

\subsubsection{Formulation of Principle II} 
\label{sec:fairness}

As mentioned in Section~\ref{sec:intro}, we propose the individual fairness constraint (Principle II) to guide the client selection across multiple rounds. It asserts that clients with similar data distributions should have similar frequencies of being selected. 
%To formulate the fairness constraint, we establish metrics to measure the similarity between a pair of clients in terms of their data distributions and the frequencies of being selected, respectively. 

%Considering the unavailability of original data distributions in federated learning, we measure the similarity between clients in the gradient space. 
First, we use $\text{Dist}_{i,j}(t) $ in Eq.~\eqref{equ:simi} to measure the similarity of the data distribution.
Second, let $x_{i,t}$ represent whether the $i$-th client is selected in the $t$-th round, with $x_{i,t}=1$ if $i\in \mathbb{S}^{t}$ and $x_{i,t}=0$ otherwise. We use 
\begin{align}
\label{equ:probability}
p_i = \frac{1}{T} \sum_{t=1}^T \mathbb{E}(x_{i,t}).
\end{align}
to evaluate the frequency of a client $i$ being selected in $T$ rounds. Note that $\{x_{i,t}\}$ forms a stochastic process, and the expectation operation $\mathbb{E}$ is applied in Eq.~\eqref{equ:probability}. %Given two clients $i,j\in\mathbb{N}$, we use $|p_i-p_j|$ to quantify the difference between their frequencies of being selected.

Building on $\text{Dist}_{i,j}(t)$ and $p_i$, we employ the $\epsilon\text{-}\delta$-IF framework utilized in~\cite{john2020verifying, Benussi2022individual} to quantify the individual fairness constraint. We evaluate a client selection strategy at the end of the $T$-th round. Specifically, given $\epsilon, \delta \geq 0$, a client selection strategy is of individual fairness if for any clients $i,j\in \mathbb{N}$ with $\text{Dist}_{i,j}(T)\leq \epsilon$, the difference in their selection frequencies, namely, $p_i$ and $p_j$, should not exceed $\delta$. The mathematical formulation of $\epsilon\text{-}\delta$-IF is provided in Definition~\ref{def:IF}, and the impact of parameter selection for $\epsilon$ and $\delta$ is discussed in Section \ref{sec:parameter_selection}. 
\begin{definition}
\label{def:IF}
($\epsilon\text{-}\delta$-IF). Consider $\epsilon, \delta\geq 0$, a client selection strategy is said to be of individual fairness if 
\begin{align}
\label{equ:IF}
% \forall\, i,j\in \mathbb{N}, \quad \text{s.t.}\,\,\text{Dist}(i,j)\leq \epsilon \,\, \Rightarrow \,\, |p_i(T) -p_j(T)|\leq \delta.
\forall\, i,j\in \mathbb{N}, \,\, \text{Dist}_{i,j}(T)\leq \epsilon \,\, \Rightarrow \,\, |p_i -p_j|\leq \delta.
\end{align}
\end{definition}

The $\epsilon\text{-}\delta$-IF requires examining Eq.~\eqref{equ:IF} for all pairs of clients, leading to significant computational overhead. To address this issue, we propose to determine a reference client $i^\star$ for each client $i$, which exhibits the largest difference in the selection frequency and has the gradient distance less than $\epsilon$. That is, 
\begin{align}
\label{equ:client_star}
     i^{\star} = \text{argmax}_{\text{Dist}_{i,j}(T)\leq \epsilon}\,|p_i-p_j|, \; \forall \; j\in \mathbb{N}.
\end{align}
Then, the $\epsilon\text{-}\delta$-IF is simplified as 
\begin{align}
\label{equ:client_star_IF}
    |p_i-p_{i^{\star}}|\leq \delta, \quad \forall \; i\in \mathbb{N}.
\end{align} 
By introducing the reference client in Eq.~\eqref{equ:client_star}, the evaluation of $\epsilon\text{-}\delta$-IF is simplified from examining Eq.~\eqref{equ:IF} for pairs of clients $i,j\in \mathbb{N}$ to evaluating Eq.~\eqref{equ:client_star_IF} for individual clients $i\in \mathbb{N}$. This reduction in the number of variables facilitates further optimization. 

\subsubsection{The Optimization Function} 
\label{sec:optimization_function}
Building on the estimation error in Eq.~\eqref{equ:diversity} and the individual fairness constraint in Eq.~\eqref{equ:client_star_IF}, we formulate the client selection strategy as an optimization problem. The objective is to select a series of subsets $\{\mathbb{S}^1, \cdots, \mathbb{S}^{T}\}$ with $|\mathbb{S}^t|=K$ that minimize the expected estimation error over all $T$ rounds while adhering to the individual fairness constraint. That is, 
\begin{align}
\label{equ:p1}
\text{(P1)} \quad   &\min_{\{\mathbb{S}^1, \cdots, \mathbb{S}^{T}\} } \; \lim_{T \to +\infty}\frac{1}{T}\, \sum_{t=1}^T \mathbb{E}\left[\text{DUB}(\mathbb{S}^t)\right], \\ \notag
    &\qquad  \textrm{s.t.} \quad \quad \text{Eq. \eqref{equ:client_star_IF}}. 
    %, \; and }\; |\mathbb{S}^t|=K.
\end{align}

However, directly solving the optimization function in Eq.~\eqref{equ:p1} is infeasible. The optimization objective and constraints, particularly $p_i$ in Eq.~\eqref{equ:client_star_IF}, are presented in a time-averaged form. That is, these values are determined by averaging over all $T$ rounds, which can only be accomplished at the end of training. In contrast, federated learning requires clients to be selected online in each round. This misalignment poses a practical implementation challenge. In our work, the solution to address this misalignment is proposed in Section \ref{sec:transformation}.

\subsection{Transformation Under Lyapunov Optimization}
\label{sec:transformation}
To solve P1 in Eq.~\eqref{equ:p1}, we leverage Lyapunov optimization, a technique from control theory used to analyze the stability of dynamic systems~\cite{diehl2010lyapunov, huang2020efficiency}. The main idea behind Lyapunov optimization is to break down the long-term time-averaged constraints into constraints that can be adhered to in each communication round.  
By leveraging Lyapunov optimization, the problem stated as P1 in Eq. \eqref{equ:p1} is ultimately converted into the problem P3 in Eq. \eqref{equ:P3}. Details of the problem transformation are described below, which consists of four steps.

\textbf{(a) Transformation of Individual Fairness Constraints. } 
Before employing Lyapunov optimization, we remove the absolute value sign from the constraint in Eq. \eqref{equ:client_star_IF} and rephrase it as two equivalent constraints,
\begin{align}
\label{equ:IF_Z} \frac{1}{T}\sum_{t=1}^T \mathbb{E} (x_{i,t}-x_{i^\star,t}) - \delta \leq 0, \;
    \forall\, i\in \mathbb{N}, \\ 
\label{equ:IF_Q} \frac{1}{T}\sum_{t=1}^T \mathbb{E} (-x_{i,t}+x_{i^\star,t}) - \delta \leq 0, \;
    \forall\, i\in \mathbb{N}.
\end{align}

\textbf{(b) Introduction of Virtual Queues $Z_i(t)$, $Q_i(t)$ and $\Theta(t)$. } Following the general process of Lyapunov optimization~\cite{diehl2010lyapunov}, we define two virtual queues, namely $Z_i(t)$ for the constraint in Eq. \eqref{equ:IF_Z} and $Q_i(t)$ for the constraint in Eq. \eqref{equ:IF_Q}, for all clients $i\in\mathbb{N}$. Specifically, these queues are initialized as $Z_i(0)=0$ and $Q_i(0)=0$, and updated according to the following rule
\begin{align}
\label{equ:update}
    Z_i(t+1) = \max \left\{Z_i(t)+x_{i,t}-x_{i^\star,t}-\delta, 0\right\}, \\  \notag
    Q_i(t+1) = \max \left\{Q_i(t)-x_{i,t}+x_{i^\star,t}-\delta, 0\right\}.
\end{align}
By the introduction of $Z_i(t)$ and $Q_i(t)$, we have Theorem \ref{theo:virtual_queue}, which converts the long-term constraints on $x_{i,t}$ in Eq. \eqref{equ:IF_Z} and \eqref{equ:IF_Q} into the stability constraints for $Z_i(t)$ and $Q_i(t)$, respectively. The proof of Theorem \ref{theo:virtual_queue} is provided in the supplementary file.
\begin{theorem}
\label{theo:virtual_queue}
    The constraints in Eq. \eqref{equ:IF_Z} and \eqref{equ:IF_Q} hold if $Z_i(t)$ and $Q_i(t)$ remain stable, that is,
    \begin{align}
        \lim_{T\to +\infty} \frac{\mathbb{E}[Z_i(T)]}{T}=0, \; \text{and} \; \lim_{T\to +\infty} \frac{\mathbb{E}[Q_i(T)]}{T}=0.
    \end{align}
\end{theorem}
Then, we define a global queue $\Theta(t)$, which stores the state of $Z_i(t)$ and $Q_i(t)$ for all clients. That is, 
\begin{align}
    \label{equ:layponov}
    \Theta(t) \triangleq \left[Z_1(t),\cdots, Z_n(t), Q_1(t), \cdots, Q_n(t)\right]. 
\end{align}

\textbf{(c) Introduction of Lyapunov Function $L\left(\Theta(t)\right)$ and Lyapunov Drift $\Delta\left(\Theta(t)\right)$. } 
Following the general process of Lyapunov optimization~\cite{diehl2010lyapunov}, we define the Lyapunov function \cite{diehl2010lyapunov} as 
\begin{align}
    L\left(\Theta(t)\right) \triangleq \frac{1}{2} \sum_{i=1}^N \left[Z_i^2(t)+Q_i^2(t)\right],
\end{align}
which represents the sum of the squares of all elements in $\Theta(t)$. Then, the increase of $\Theta(t)$ from the communication round $t$ to $(t + 1)$ is formulated as
\begin{align}
\label{equ:drift}
    \Delta\left(\Theta(t)\right) \triangleq \mathbb{E} \left[L\left(\Theta(t+1)\right)- L\left(\Theta(t)\right)|\Theta(t)\right],
\end{align}
which is called \emph{Lyapunov drift} \cite{diehl2010lyapunov}. If the drift $\Delta\left(\Theta(t)\right)$ remains sufficiently small in each round, the constraints in Eq. \eqref{equ:IF_Z} and Eq. \eqref{equ:IF_Q} will be satisfied after $T$ rounds. That is, introducing $\Delta\left(\Theta(t)\right)$ allows us to break down the time-averaged constraints in Eq. \eqref{equ:IF_Z} and Eq. \eqref{equ:IF_Q} into the specific requirements for $\Delta\left(\Theta(t)\right)$ in each communication round.

\textbf{(d) Introduction of the Tradeoff Factor $V$ and two Variables $m_{i,t}$ and $ n_{i,t}$. } To ensure  $\Delta\left(\Theta(t)\right)$ remains sufficiently small, a straightforward method is to combine the optimization objection with the drift $\Delta\left(\Theta(t)\right)$, and minimize both simultaneously. Mathematically, in our work, the optimization problem in Eq. \eqref{equ:p1} is transformed into 
\begin{align}
\label{equ:P2}
    \text{(P2)}\; \min_{\mathbb{S}^t}\quad (1-V)\cdot \Delta\left(\Theta(t)\right)+V\cdot \text{DUB}(\mathbb{S}^t)
\end{align}
with $|\mathbb{S}^t|=K$ as a constraint, where $V$ is a predefined trade-off factor.
Note that the expectation notation $\mathbb{E}$ is dropped since we are only concerned with a single communication round.

However, solving P2 in Eq. \eqref{equ:P2} still presents two challenges.  First, determining $i^\star$ requires the information about $p_i$, which is not available until the end of training. To address the issue, we calculate the frequency of a client $i$ being selected up to the $t$-th round, denoted as $p_i(t)=\frac{1}{t}\sum_{k=1}^t x_{i,t}$. Then, we determine the reference client in the $t$-th round by 
\begin{align}
\label{equ:client_star_t}
i^\star_t=\text{argmax}_{\text{Dist}_{i,j}(t)\leq \epsilon}|p_i(t)-p_j(t)|, \; \forall \, j\in \mathbb{N}.
\end{align}
In our experiment, we replace the reference client $i^\star$ in Eq. \eqref{equ:client_star} with $i^\star_t$ in Eq. \eqref{equ:client_star_t}, thus it can be determined using information available up to the $t$-th round.

Additionally, because the computation of $\Delta\left(\Theta(t)\right)$ requires information about $L\left(\Theta(t+1)\right)$, which is not available in the current round $t$. To overcome this issue, we derive an upper bound for $\Delta\left(\Theta(t)\right)$ using Theorem \ref{theo:drift_bound} and minimize the upper bound instead. The proof is provided in the supplementary file.
\begin{theorem}
\label{theo:drift_bound}
    Define 
    \begin{align}
        \label{equ:m_n}
         m_{i,t}=x_{i,t}-x_{i^\star_t,t}-\delta, \;  \text{and} \; n_{i,t}=-x_{i,t}+x_{i^\star_t,t}-\delta,
    \end{align}
   then $\Delta\left(\Theta(t)\right)$ is bounded by
    \begin{align}
        \Delta\left(\Theta(t)\right)\leq B + \sum_{i=1}^N \left[Z_i(t)m_{i,t} + Q_i(t)n_{i,t}\right],
    \end{align}
    where $B$ is a constant.
\end{theorem}

\textbf{The Final Formulation.} Based on Theorem \ref{theo:drift_bound}, the problem P2 in Eq. \eqref{equ:P2} is transformed into 
\begin{align}
\label{equ:P2-1}
\min_{\mathbb{S}^t}\; (1-V)\sum_{i=1}^N \left[Z_i(t)m_{i,t} + Q_i(t)n_{i,t} \right]+V\cdot \text{DUB}(\mathbb{S}^t). 
\end{align}
Substituting $\text{DUB}(\mathbb{S}^t)$ by Eq. \eqref{equ:diversity_bound} and moving the sum sign outside of the minimum sign, the problem in Eq. \eqref{equ:P2-1} is equivalent to
\begin{align}
\label{equ:P3}
&\min_{\mathbb{S}^t}\; G(\mathbb{S}^t)=\sum_{i=1}^N \min_{j\in \mathbb{S}^t } \Big\{(1-V)\big[Z_i(t)m_{i,t}+ Q_i(t)n_{i,t}\big]  \\ \notag
&\text{(P3)}\qquad\qquad \qquad\qquad +V\cdot \big\Vert\nabla f_i(\boldsymbol{w}^t) - \nabla f_{j}(\boldsymbol{w}^t) \big\Vert \Big\}, 
\end{align}
which is the final optimization function. In the following, we select clients by minimizing $G(\mathbb{S}^t)$ in Eq. \eqref{equ:P3} in each round.

\section{The Proposed \emph{LongFed}}
\label{sec:strategy}

\subsection{The Client Selection Strategy}
The optimization problem in Eq. \eqref{equ:P3} is NP-hard as it involves calculating the value of $G(\mathbb{S}^t)$ for $\frac{N!}{K!(N-K)!}$ subsets, where $! $ denotes the factorial function \cite{krause2014submodular}. To address this issue, we exploit the submodular nature of $G(\mathbb{S}^t)$.

\begin{algorithm}[tb]
   \caption{The Proposed Client Selection Strategy}
   \label{alg:strategy}
\begin{algorithmic}[1]
   \State {\bfseries Input:} $Z_i(t)$, $Q_i(t)$
   \State {\bfseries Output:} the selected subset $\mathbb{S}^t$ %and their weights $\{\theta_j\}$
   \State Initialize $\mathbb{S}_0^t=\emptyset$, $\mathbb{P}_0=\mathbb{N}$, and $e=1$.
   \For{ $k\in [0,K-1]$}
    \State Determine the reference client $i^\star$ for $i\in \mathbb{P}_{k}$
    \State Compute $m_{i,t}$ and $n_{i,t}$ for $i\in \mathbb{P}_{k}$
    \State Calculate $G(\mathbb{S}_{k}^t\cup \{i\}), \forall\, i\in \mathbb{P}_{k}$
    \State Identify the client $i_{\text{max}}=\text{argmax}_i \overline{G}(\mathbb{S}_{k}^t\cup \{i\})$
    \State $\mathbb{S}_{k+1}^t \leftarrow \mathbb{S}_{k}^t\cup i_{\text{max}}, \; \mathbb{P}_{k+1} \leftarrow \mathbb{P}_{k}\backslash i_{\text{max}}$
   \EndFor
   % \FOR{each selected client $j \in \mathbb{S}^t$}
   % \STATE compute the weight $\theta_j$
   % \ENDFOR
\end{algorithmic}
\end{algorithm}

Specifically, a set function $g:2^\mathbb{N}\rightarrow \mathbb{R}$ is submodular if for every $A\subseteq B\subseteq \mathbb{N}$ and $i\in \mathbb{N} \backslash B$ it holds that $g(A\cup \{i\})-g(A) > g(B\cup \{i\})-g(B)$. One typical example of submodular function is the facility location function \cite{cornuejols1977uncapacitated}. Suppose we aim to select locations from a set of positions $\mathbb{N}=\{1,\ldots,N\}$ to open facilities and to serve a collection of $K$ users. If a facility is located at position $j$, the service it provides to user $i$ is quantified by $M_{i,j}$. Each user is assumed to select the facility with the highest service, and the total service provided to all users is modeled by the set function
\begin{align}
f(\mathbb{S}) = \sum_{i=1}^m \max_{j\in \mathbb{S}} M_{i,j},
\end{align}
where $f(\emptyset)=0$. If $M_{i,j}\geq 0$ for all $i,j$, then $f(\mathbb{S})$ is a monotone submodular function. By introducing an auxiliary element $e$, the set function $G(\mathbb{S}^t)$ in Eq. \eqref{equ:P3} is transformed into 
\begin{align}
    \label{equ:G_s0}
    \overline{G}(\mathbb{S}^t) = G(\{e\}) - G(\mathbb{S}^t\cup \{e\}),
\end{align}
which is a facility location function and has a submodular nature. $\overline{G}(\mathbb{S}^t)$ measures the decrease in the value of $G(\mathbb{S}^t)$ associated with the set $\mathbb{S}^t$ compared to that associated with just the auxiliary element $e$. Without loss of generality, we set the auxiliary element as $e=1$. Consequently, minimizing $G(\mathbb{S}^t)$ in Eq. \eqref{equ:P3} is equivalent to maximizing $\overline{G}(\mathbb{S}^t)$ in Eq. \eqref{equ:G_s0}.

Prior studies show that the greedy algorithm is an effective solution for finding the maximum value of a submodular function \cite{wolsey1982analysis,krause2014submodular}. Following this greedy algorithm, we propose a strategy to select clients in the $t$-th round, as outlined in Algorithm \ref{alg:strategy}. The proposed strategy starts with an empty set $\mathbb{S}_0^t=\emptyset$, and initializes a candidate set as $\mathbb{P}_0=\mathbb{N}$ (Line 3). In an iteration $k\in [0,K-1]$, we first determine the reference client $i^\star_t$ for client $i\in \mathbb{P}_{k}$, and compute $m_{i,t}$ and $n_{i,t}$ with $x_{i,t}=1$ if $i\in \mathbb{S}^t_{k}$, and similarly for $x_{i^\star_t,t}$ (Line 5-6). Next, we calculate $\overline{G}(\mathbb{S}_{k}^t\cup \{i\})$ for all clients $i\in \mathbb{P}_{k}$, and identify the client $i_{\text{max}}$ with the maximum value (Line 7-8). Subsequently, the client $i_{\text{max}}$ is removed from $\mathbb{P}_{k}$ and added to the subset $\mathbb{S}_{k}^t$ (Line 9). This iteration continues until $K$ clients are selected. The complexity of the algorithm is $O(KN)$.

Based on the client selection strategy in Algorithm \ref{alg:strategy}, we further present the proposed federated training algorithm, as outlined in Algorithm \ref{alg:FL}. 
%The algorithm requires four parameters as input, including $\epsilon$, $\delta$ in $\epsilon\text{-}\delta$-IF in Eq. \eqref{equ:IF}, the trade-off factor $V$, and the total communication rounds $T$. 
First, the server initializes the queues $Z_i(0)$ and $Q_i(0)$, along with the model parameter $\boldsymbol{w}^0$ (Line 3). Subsequently, the server selects the subset of clients $\mathbb{S}^t$ (Line 5-9). If the communication round $t=0$, all clients are selected; otherwise, the server selects $K$ clients according to Algorithm \ref{alg:strategy}. The server then sends the model parameter $\boldsymbol{w}^t$ to the selected clients, and these clients perform local training and send the gradients back to the server (Line 10-11). The server updates $Z_i(t)$ and $Q_i(t)$ according to Eq. \eqref{equ:update}, and aggregates these results to obtain the model parameter $\boldsymbol{w}^{t+1}$ according to Eq. \eqref{equ:aggregate}. The iteration is repeated until completing the $T$ rounds. 

%The auxiliary element can be any client in the set $\mathbb{N}$. Without loss of generality, we set the auxiliary element as $e=1$ in our work.

\begin{algorithm}[tb]
   \caption{The Federated Training Algorithm}
   \label{alg:FL}
\begin{algorithmic}[1]
   \State {\bfseries Input:} $\epsilon$, $\delta$, $V$ and $T$
   \State {\bfseries Output:} The trained model $\boldsymbol{w}^T$
   \State Initialize $\boldsymbol{w}^0$, $Z_i(0)=Q_i(0)=0$, 
   \For{$t\in [0,T]$}
   \If{$t=0$}
   \State Select all clients with $\mathbb{S}^t=\mathbb{N}$
   \Else
   \State Select K clients $\mathbb{S}^t$ according to Algorithm \ref{alg:strategy}
   \EndIf
   \State The server sends $\boldsymbol{w}^t$ to the selected clients in $\mathbb{S}^t$
   \State Clients train local models in parallel and send the gradients $\nabla f_i(\boldsymbol{w}^t)$ to the server
   \State The server update $Z_i(t+1)$, $Q_i(t+1)$
   \State The server aggregate the results and obtain $\boldsymbol{w}^{t+1}$
   \EndFor
\end{algorithmic}
\end{algorithm}

\subsection{Convergence Analysis} 
%Under the assumptions that $f_1, \ldots, f_N$ are all $L$-smooth and strongly convex, and that the variance and squared norm of their gradients are bounded, 
% We provide a theoretical analysis to demonstrate the convergence ability of the proposed client selection strategy. 
To implement the theoretical analysis, we establish six assumptions regarding the local models and data distribution heterogeneity among clients. The analysis uses FedAvg as the aggregation method, and it can be extended to other federated optimization methods as well. 

First, we assume that the estimation error between the client subset and the full client set (Eq. \eqref{equ:diversity}) is small and can be quantified by a variable $\rho$, as stated in Assumption 1.  Note that $\rho$ is used as a measure to characterize the quality of the estimation, and the analysis holds for any $\rho<\infty$. 

\textbf{Assumption 1}. At a round $t$, we assume that the gradient aggregated from the selected subset of clients can provide a good approximation of the gradient aggregated from the full set, i.e.,
\begin{align}
    \left\Vert\sum_{i\in \mathbb{N}} \nabla f_i(\boldsymbol{w}^t) - \sum_{j\in \mathbb{S}^t } \theta_j^t \nabla f_j(\boldsymbol{w}^t) \right\Vert \leq \rho.
\end{align}

Next, we outline the assumptions regarding local models $f_1,\cdots, f_N$ and their gradients $\nabla f_1(\boldsymbol{w}^t), \cdots, \nabla f_N(\boldsymbol{w}^t)$, as stated in Assumption 2-5. These assumptions are standard and widely used in the federated optimization literature \cite{li2020convergence,cho2020client,tang2022fedcor,balakrishnan2022diverse}. 
%Assumption 2 and 3 are standard, and typical examples including the logistic regression and softmax classifier. 

\textbf{Assumption 2}. $f_1,\cdots, f_N$ are all $L$-smooth. Formally, for all $\boldsymbol{v}$ and $\boldsymbol{w}$, we have 
\begin{align}
f_k(\boldsymbol{v})\leq f_k(\boldsymbol{w})+(\boldsymbol{v}-\boldsymbol{w})^T \nabla f_k(\boldsymbol{w}) + \frac{L}{2} \Vert \boldsymbol{v}-\boldsymbol{w} \Vert _2^2.
\end{align}

\textbf{Assumption 3}. $f_1,\cdots, f_N$ are all $\mu$-strongly convex. Formally, for all $\boldsymbol{v}$ and $\boldsymbol{w}$, it holds
\begin{align}
f_k(\boldsymbol{v})\geq f_k(\boldsymbol{w})+(\boldsymbol{v}-\boldsymbol{w})^T \nabla f_k(\boldsymbol{w}) + \frac{\mu}{2} \Vert \boldsymbol{v}-\boldsymbol{w} \Vert _2^2.
\end{align}

\textbf{Assumption 4}. The variance of gradients for $f_i$ is bounded for $i\in \mathbb{N}$. Formally, letting $\alpha_i^{t}$ be a data sample randomly chosen from the local dataset of the client $i$, we have
\begin{align}
\mathbb{E}\left[\Vert \nabla f_i(\boldsymbol{w}_i^t, \alpha_i^t)-\nabla f_i(\boldsymbol{w}_i^t)\Vert\right] \leq B_2.
\end{align}

\textbf{Assumption 5}. The expected squared norm of gradients is uniformly bounded, that is,
\begin{align}
\mathbb{E}\left[\Vert \nabla f_i(\boldsymbol{w}_i^t, \xi_i^t) \Vert \right]\leq B_3.
\end{align}

Furthermore, we introduce the term $\Gamma$ in Assumption 6 to quantify the data heterogeneity among clients. If the data distribution among clients is independently and identically distributed, then $\Gamma$ approaches zero as the number of clients grows. Conversely, if the data distribution is heterogeneous, the magnitude of $\Gamma$ reflects the degree of heterogeneity. 

\textbf{Assumption 6}. Let $f^\star$ and $f_i^\star$ be the minimum values of $f$ and $f_i$, respectively. We consider the degree of data heterogeneity to be bounded, that is,
\begin{align}
\Gamma = \Vert f^\star - \sum_{i=1}^N \theta_i f_i^\star \Vert \leq B_4.
\end{align}

Based on these assumptions, the proposed client selection strategy is demonstrated to converge to the global optimal parameter $\boldsymbol{w}^\star$ at a rate of $\mathcal{O}(1/t)$ for heterogeneous data settings, as stated in Theorem \ref{theo:convergence}. The proof is provided in the supplementary file. The convergence rate $\mathcal{O}(1/t)$ is the same as loss-guided client selection methods \cite{balakrishnan2022diverse,tang2022fedcor,cho2020client}.
 
\begin{theorem}
\label{theo:convergence}
Under Assumptions 1-6, we have 
\begin{align}
    \mathbb{E} \Vert \boldsymbol{w}^\star -\boldsymbol{w}^t\Vert_2^2 \leq \mathcal{O}(1/t) + \mathcal{O}(\rho).
\end{align}
\end{theorem}

In Eq.~\eqref{equ:diversity}, the term $\rho$ encodes the estimation error. When more clients are selected, the term $\rho$ is smaller. Particularly, $\rho$ becomes zero when all clients are selected, i.e., $K=N$. In practical settings, as only limited clients can be selected, $\rho$ remains a non-vanishing term. In our experiments, we also observe that there exists a non-diminishing solution bias dependent on $\rho$. This observation is consistent with our theoretical analysis. The impact of varying $K$ is also empirically analyzed in~\cite{li2024adafl}.

\section{Experiment}
\label{sec:test}

\subsection{Experimental Settings}
\label{sec:setting}

We consider the cross-device federated learning scenario where $N=100$ clients exist, each with limited computational power.  The FedAvg method~\cite{mcmahan2017communication} is used as the aggregation method. We evaluate our method on two datasets, FMNIST \cite{xiao2017fashion} and CIFAR-10 \cite{krizhevsky2009learning}. Following~\cite{tang2022fedcor}, we opt for basic and small-scale models to accommodate the clients' restricted computational resources. Specifically, for the FMNIST dataset, we utilize a multilayer perception (MLP) with two hidden layers. For CIFAR-10, we adopts a convolutional neural network (CNN) architecture consisting of three convolutional layers. Details of training and hyperparameter settings can be found in supplementary file. 

\textbf{Data Partition Methods.} We explore four data partitioning methods. The first is an independently and identically distributed (\textbf{IID}) approach, where we randomly partition the dataset into $N$ parts, with each client assigned one part. We also consider three heterogeneous data partitioning approaches as follows.

\textbf{(i) 1 Shard Per Client (1SPC)}. Following~\cite{mcmahan2017communication}, we divide the dataset into $N$ shards, ensuring that data within a shard shares the same label. We randomly assign a shard to each client so that a client has data with only one label. In this case, we select $K=10$ clients in each round. 

\textbf{(ii) 2 Shards Per Client (2SPC)}. The dataset is partitioned into $2N$ shards, with each shard containing data that shares the same label. Clients are randomly assigned two shards, allowing them to have data with as most two distinct labels. In this case, we set  $K=5$ in each round.

\textbf{(iii) Dirichlet Distribution (Dir)}. We partition the dataset based on a Dirichlet distribution parameterized by a concentration variable $\alpha$, where a smaller value of $\alpha$ indicates higher heterogeneity. In our work, we set $\alpha=0.8$, and determine the data size of each label for each client following~\cite{hsu2019measuring,tang2022fedcor}. In this case, we choose $K=5$ clients in each round. We also experimented with smaller values, such as $\alpha=0.5$ and $\alpha=0.1$. We found that our method consistently achieved the best model performance. 
% In this case, we choose $K=5$ clients in each round.

Note that the 1SPC and 2SPC scenarios address label shift heterogeneity, while the Dir scenario accounts for heterogeneity in both labels and dataset size.

%Furthermore, apart from the assumption of a fixed number of clients ($K$) to be selected, we also explore the impact of varying $K$ and propose an adaptive method to determine its value. Details of the adaptive method can be found in \cite{li2024adafl}.

%We conduct the experiments five times using different randomized initial conditions and present the average results. In our experiments, we set $\delta=0.1$ , $\epsilon=0.3$, and $V=0.9$ for FMNIST, and $\delta=0.05$, $\epsilon=0.05$, $V=0.7$ for CIFAR-10. 

\begin{figure}[tbp]
    \centering
    \includegraphics[width=6.5cm]{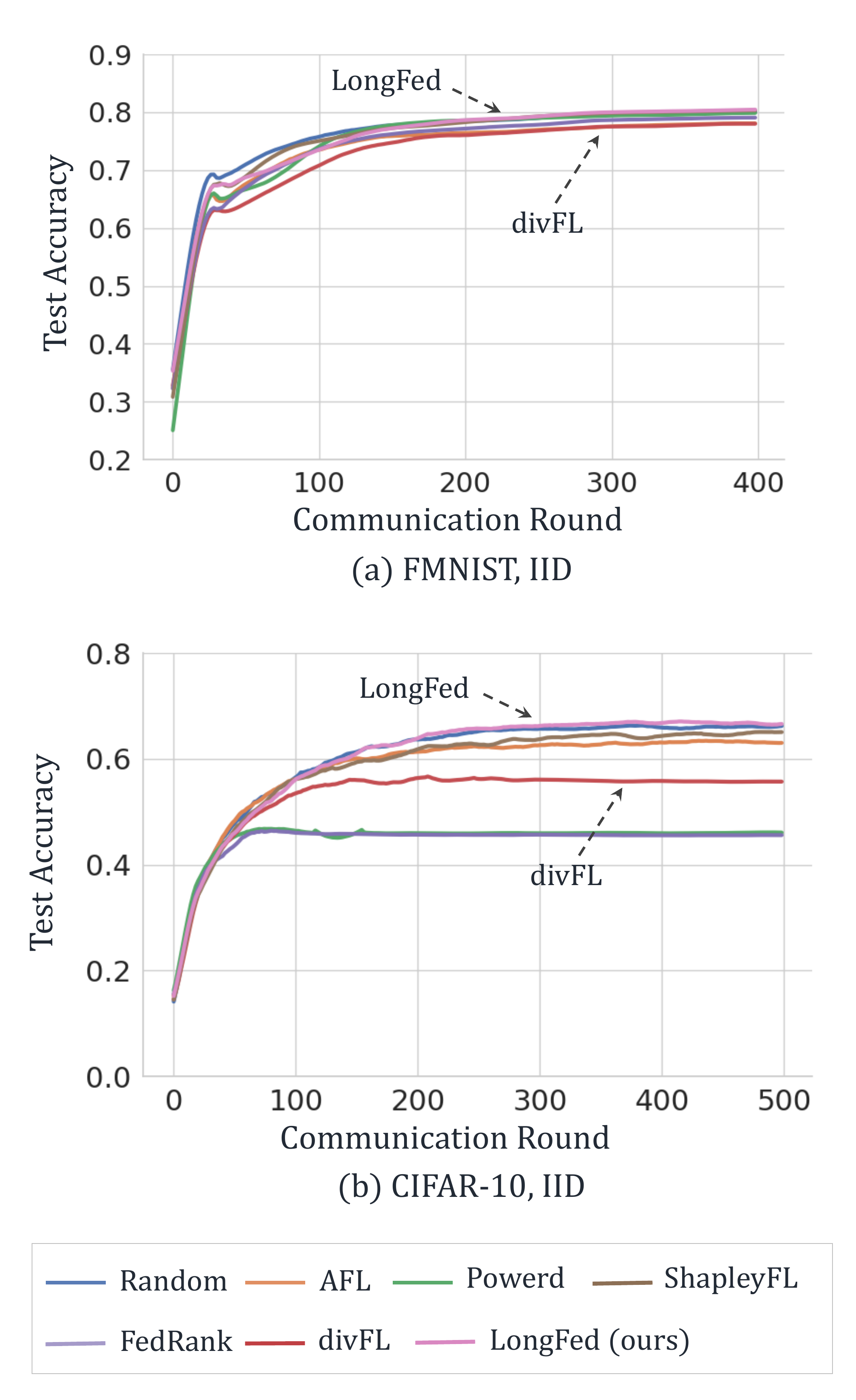}
    \caption{Test accuracy under the IID scenario.}
    \label{fig:convergence_iid}
\end{figure}

\textbf{Comparison Methods.} We first evaluate the effectiveness of the proposed method by comparing it with existing client selection strategies in terms of model performance. The baselines are as follows. (1) Random selection strategy (Random, 2017)~\cite{mcmahan2017communication}; (2) Loss-guided selection methods, including active client selection strategy (AFL, 2019)~\cite{goetz2019active}, power-of-choice selection strategy (Powerd, 2022)~\cite{cho2022towards}, and diverse client selection strategy (divFL, 2022)~\cite{balakrishnan2022diverse}; (3) Contribution-based methods, including the Shapley value-based method (ShapleyFL, 2023)~\cite{sun2023shapleyfl} and ranking-based client selection (FedRank, 2024)~\cite{tian2024ranking}. Cluster-based methods are not considered, as they require explicit clustering patterns, which are not applicable in the 2SPC and Dir scenarios. Additionally, we include full participation method (Full) with $K=100$ as a reference.

Furthermore, we evaluate the proposed method in terms of fairness. As introduced in Section \ref{sec:review}, the uniform selection constraint from \cite{shi2023fairness} (2023) in performance fairness research is most aligned with our work. In our study, we adopt this uniform selection constraint and apply it to loss-guided selection methods, given their superior model performance. Specifically, we apply the uniform selection constraint to AFL~\cite{goetz2019active}, Powerd~\cite{cho2022towards}, and divFL~\cite{balakrishnan2022diverse}, denoting them as AFL+Fair, Powerd+Fair, and divFL+Fair, respectively. We then compare our method against these three baselines in terms of fairness. 

\begin{figure*}[tbp]
    \centering
    \includegraphics[width=16cm]{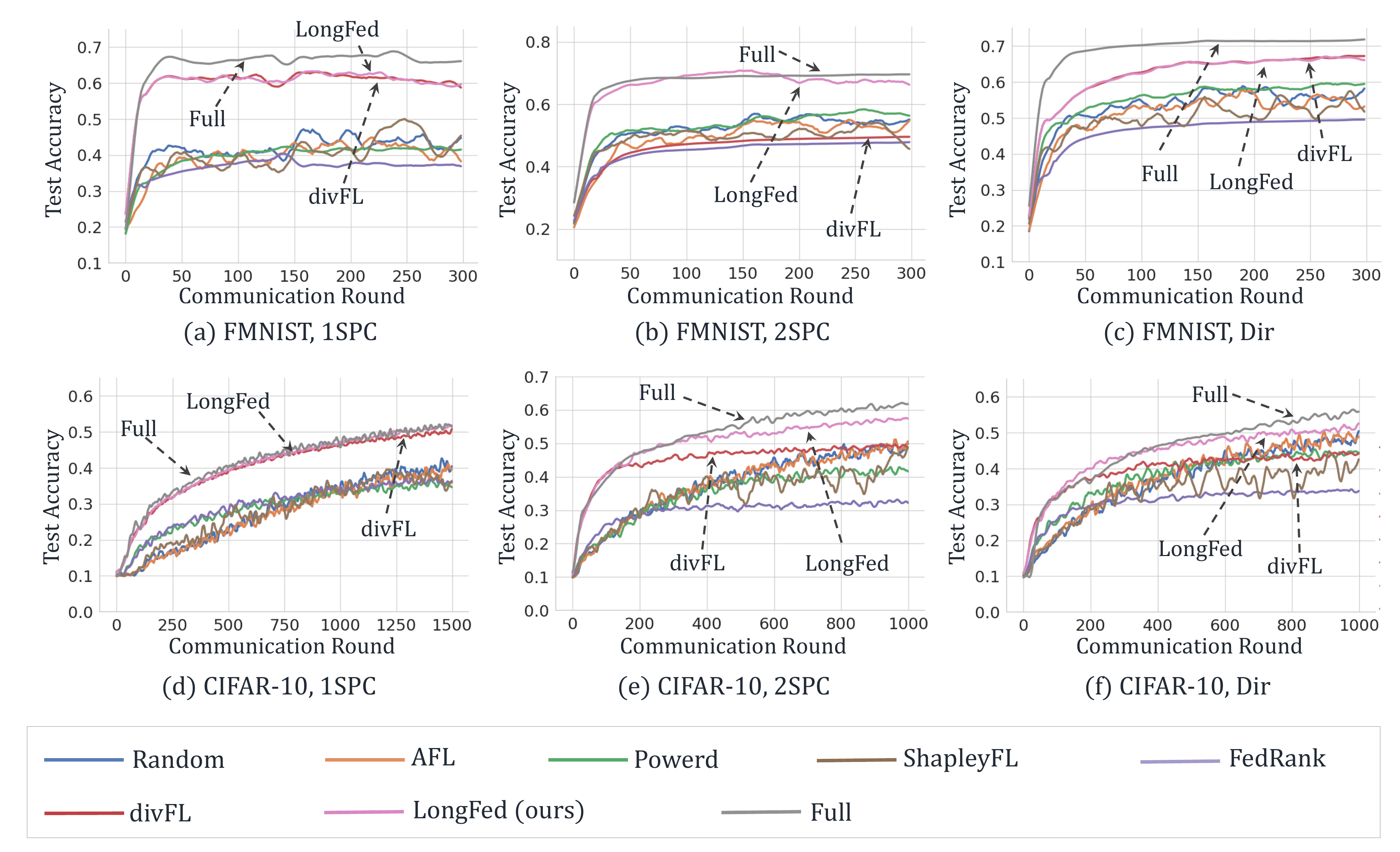}
    \caption{Test accuracy on FMNIST and CIFAR-10 under three heterogeneous data partitioning settings.}
    \label{fig:convergence_noniid}
    % \vspace{-0.1cm}
\end{figure*} 

Note that the divFL~\cite{balakrishnan2022diverse} method also selects clients to best represent full client participation, which aligns with our proposed Principle I. However, divFL does not consider fairness. By comparing divFL with our proposed  \emph{longFed}, we can evaluate the impact of our proposed individual fairness on improving model performance. Additionally, by comparing divFL+Fair with \emph{longFed}, we can further evaluate how our proposed individual fairness outperforms the uniform selection constraint. 

%We present the average results using five random seeds. FMNIST, $\delta=0.1$ and $\epsilon=0.3$, $V=0.9$. CIFAR, $\delta=0.05$, $\epsilon=0.05$, $V=0.7$. Our method is denoted as \emph{DFairFed} We compare our method with three baseline methods. including a random selection strategy (Random)\cite{mcmahan2017communication}, an active client selection strategy (AFL) \cite{goetz2019active}, where clients with large training loss have the large probability of being selected, and power-of-choice strategy (Powerd), which is inspired by the power of $d$ choices load balancing strategy that is commonly used in queueing systems \cite{cho2020client}. 

\subsection{Experimental Results}
\label{sec:exp_result}

\subsubsection{Model Performance} 

The experimental results under the IID scenario are presented in Fig.~\ref{fig:convergence_iid}. First, as shown in Fig.~\ref{fig:convergence_iid} (a) and (b), our proposed \emph{longFed} outperforms existing methods, particularly on CIFAR-10 dataset, validating its effectiveness in enhancing model performance in the IID scenario. Second, \emph{longFed} consistently outperforms divFL, further demonstrating the effectiveness of our proposed individual fairness in the IID scenario.  

\begin{figure}[tbp]
    \centering
    \includegraphics[width=7.5cm]{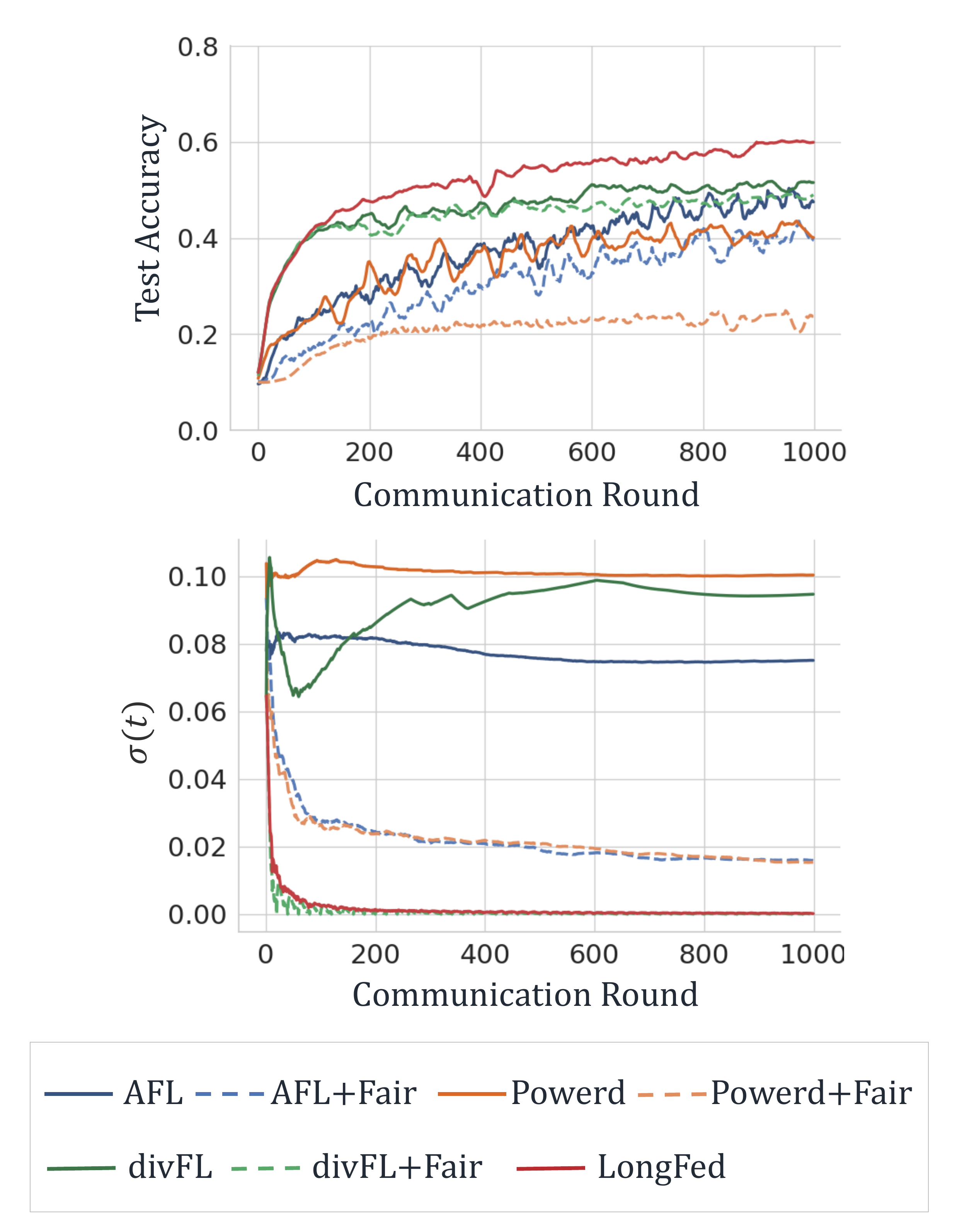}
    \caption{Fairness results on the CIFAR-10 dataset in the 2SPC scenario.}
    \label{fig:fairness_cifar}
\end{figure}

% In Fig. \ref{fig:convergence_iid} and Fig. \ref{fig:convergence_noniid}, the $x$ axis represents the communication round, and the $y$ axis displays the test accuracy of the global model. 
%For a clearer comparison, we apply the moving average rule~\cite{chiarella2006dynamic} with a sliding window size of 10 to smooth the results. 
The experimental results under three heterogeneous scenarios are presented in Fig.~\ref{fig:convergence_noniid}. First, compared to prior works (except for divFL), the proposed \emph{longFed} exhibits faster convergence and superior test accuracy. Notably, it achieves an approximate $20\%$ improvement in the 1SPC scenario and an $8\%$ improvement in the 2SPC and Dir scenarios on the FMNIST dataset. Besides, the proposed method consistently achieves performance comparable to full client participation across all three scenarios and both datasets. This validates the effectiveness of our method in enhancing model performance. Furthermore, compared to divFL, \emph{longFed} exhibits similar performance in Fig.~\ref{fig:convergence_noniid} (a), (c), and (d), but achieves significant improvements in the other three scenarios. This highlights the effectiveness of our proposed individual fairness in improving model performance. 

Notably, in Fig.~\ref{fig:convergence_noniid} (b) and (f), \emph{longFed} even surpasses full client participation in some rounds. This phenomenon typically occurs in the early stages of training due to optimization dynamics and the stochasticity introduced by client selection. As training progresses, these effects gradually diminish. Additionally, FedRank performs the worst, as it relies solely on pairwise relationships between clients and lacks a global perspective on data diversity. This leads to biased selections-for example, it tends to overlook important yet less frequently ranked clients-ultimately resulting in suboptimal training and poorer convergence.

\subsubsection{Fairness}
We evaluate the fairness through the standard deviation in the selection probability, 
%$p_i$ between clients with similar data distributions, 
denoted as $\sigma(t)$. 
To calculate $\sigma(t)$, we define $c_i(t)$ as the accumulated number of selections for client $i$ during previous $t$ rounds. That is, $c_i(t)=\sum_{k=1}^t x_{i,k}$. For each client $i$, we identify clients whose similarity with it is less than $\epsilon$, denoted as $\mathbb{I}_i(t) = \{j \in \mathbb{N} \mid \text{Dist}_{i,j}(t) < \epsilon\}$. We then compute the average selection count among clients in $\mathbb{I}_i(t)$ by
\begin{align}
\label{equ:select_avg}
\overline{c}_i(t) = \frac{1}{|\mathbb{I}_i(t)|} \sum_{j \in \mathbb{I}_i(t)} c_j(t),
\end{align}
and the standard deviation in their selection counts by
\begin{align}
\label{equ:select_std}
\sigma(t) = \sqrt{\frac{1}{N}\sum_{i=1}^N \left[c_i(t) - \overline{c}_i(t)\right]^2}.
\end{align}
A smaller $\sigma(t)$ indicates that the client selection strategy aligns more closely with the individual fairness constraint. 

We compare the proposed method with three baselines: AFL~\cite{goetz2019active}, Powerd~\cite{cho2022towards}, and divFL~\cite{balakrishnan2022diverse}, along with their respective versions incorporating the uniform selection constraint, denoted as AFL+Fair, Powerd+Fair, and divFL+Fair. These methods are evaluated based on two key aspects: test accuracy, and the standard deviation $\sigma(t)$. The analysis is conducted on both the CIFAR-10 and FMNIST datasets under the 2SPC scenario, with similar trends observed in other settings. The results of the CIFAR-10 and FMNIST datasets are illustrated in Fig.~\ref{fig:fairness_cifar} and Fig.~\ref{fig:fairness_fmnist}, respectively. In these figures, the vanilla methods are represented by solid lines, while the methods incorporating the uniform selection constraint are indicated by dashed lines. 

\begin{table*}[tbp]
  \centering
  \caption{Time overhead analysis on FMNIST in the 1SPC scenario.}
    \begin{tabular}{cccccccc}
    \toprule
    Method & Random & AFL   & Powerd & ShapleyFL & FedRank &divFL & LongFed \\
    \midrule
    Time (ms) for client selection &  0.038     &  0.153    &  0.016  & 221.023  &0.057 &0.219  &  0.529 \\
    Time (ms) for a complete round &  2019     &   2044    &  2033   &3242  &2045  &2407   &  2472  \\
    \bottomrule
    \end{tabular}%
  \label{tab:run_time}%
\end{table*}%

In Fig.~\ref{fig:fairness_cifar}, the test accuracy results indicate that the proposed method achieves the best performance. However, applying the uniform selection constraint leads to a degradation in model performance for all three baselines. On the other hand, from the perspective of fairness, incorporating the uniform selection constraint reduces the standard deviation  $\sigma(t)$ for all three baseline methods. Notably, divFL+Fair achieves the lowest $\sigma(t)$ among all methods. The proposed methods attains a comparable but slightly higher standard deviation than divFL+Fair. This phenomenon occurs because the uniform selection constraint enforces equal selection probabilities across all clients, whereas the proposed method ensures similar selection probabilities only for clients with similar data distributions. As a result, divFL+Fair exhibits slightly better fairness performance than our method. Considering both test accuracy and fairness (as measured by $\sigma(t)$), the proposed individual fairness approach demonstrates a superior balance, effectively improving both model performance and fairness simultaneously compared to the uniform selection constraint. 

\begin{figure}[tbp]
    \centering
    \includegraphics[width=7.5cm]{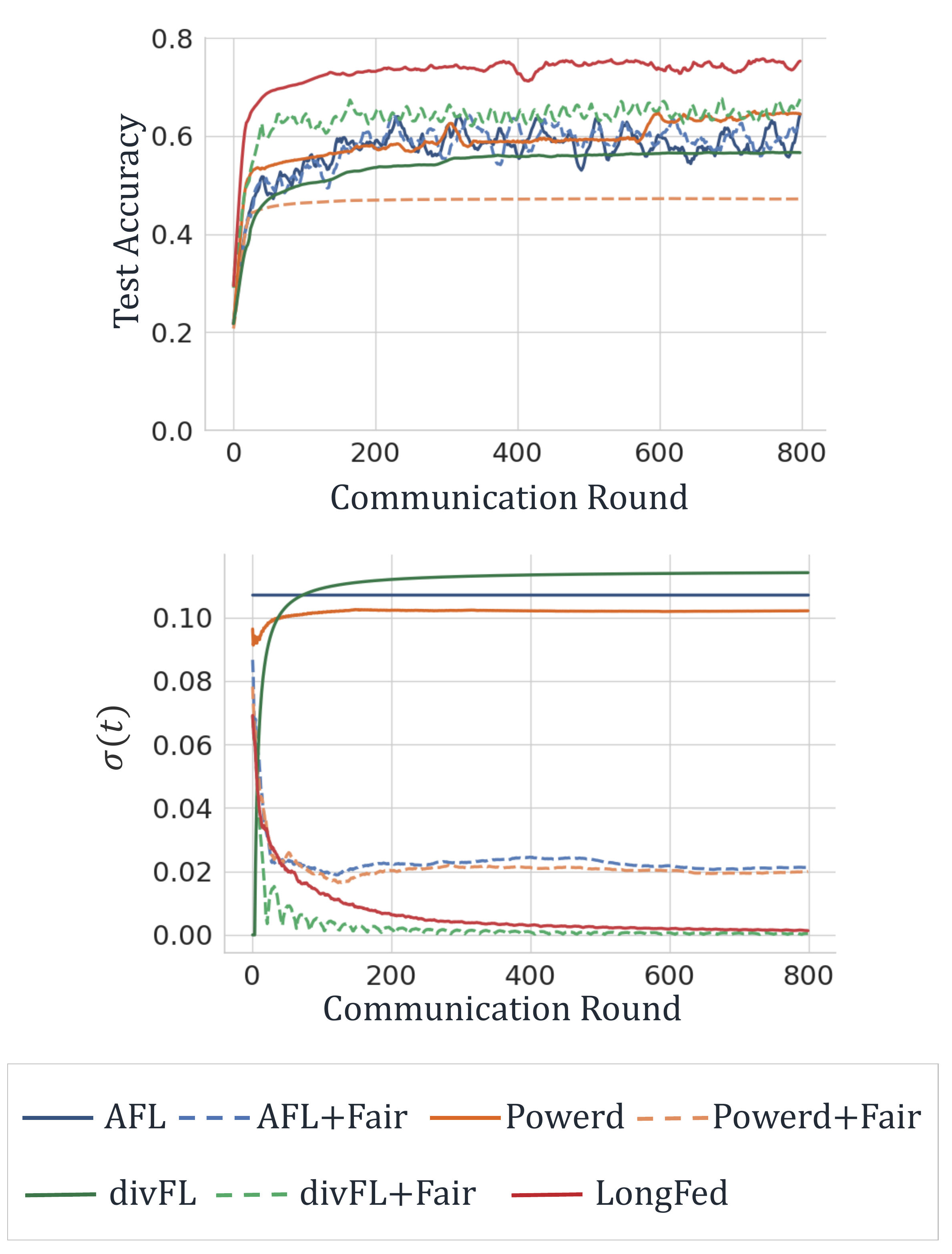}
    \caption{Fairness results on the FMNIST dataset in the 2SPC scenario. }
    \label{fig:fairness_fmnist}
\end{figure}

In Fig.~\ref{fig:fairness_fmnist}, the results for the standard deviation $\sigma(t)$ follow a similar trend to that of the CIFAR-10 dataset. That is, introducing the uniform selection constraint improve fairness, and the proposed \emph{LongFed} achieves fairness comparable to divFL+Fair. For test accuracy, we observe that applying the uniform selection constraint leads to a decline in performance for Powerd, while AFL maintains similar performance, and divFL exhibits a slight improvement. This may be because these methods typically select the same subset of clients across multiple rounds, potentially leading the system to a suboptimal solution. The introduction of the uniform selection constraint forces these methods to select different clients, particularly those that were previously under-selected, helping to escape the suboptimal state. Most importantly, the proposed \emph{LongFed} still achieves the best overall performance, demonstating its effectivess in balancing both fairness and model accuracy.

\subsubsection{Time Overhead Analysis}

We evaluate the time overhead of the proposed client selection strategy, and the results are presented in Table \ref{tab:run_time}. 
The analysis is conducted on the FMNIST dataset under the 1SPC scenario, with similar trends observed across other scenarios. 
In Table~\ref{tab:run_time}, the first row highlights the time dedicated solely to client selection, while the second row denotes the overall time required for a complete round, including client selection, local updates, and global aggregation. 

First, examining the time specifically for client selection, our proposed method exhibits only a marginal increase (less than 0.4ms) compared to existing methods. This slight increase is primarily due to the computational cost of evaluating the distance $\text{Dist}_{i,j}(t)$ in Eq.~\eqref{equ:simi}. However, this can be mitigated by employing more efficient distance computation techniques in high-dimensional gradient space. More importantly, the time required for client selection (approximately 0.5ms) is negligible compared to the total time required to complete a round (approximately 2000ms). Therefore, the marginal increase in time for the proposed strategy is justified, given its superior improvements in model performance and fairness. 

Additionally, ShapleyFL~\cite{sun2023shapleyfl} incurs significantly higher selection time due to the computationally expensive process of  Shapley values, which involves combinatorial computations. 

\subsection{Visualization of Client Selection Strategy}
\label{sec:visual}

\begin{figure*}[tbp]
    \centering
    \includegraphics[width=18cm]{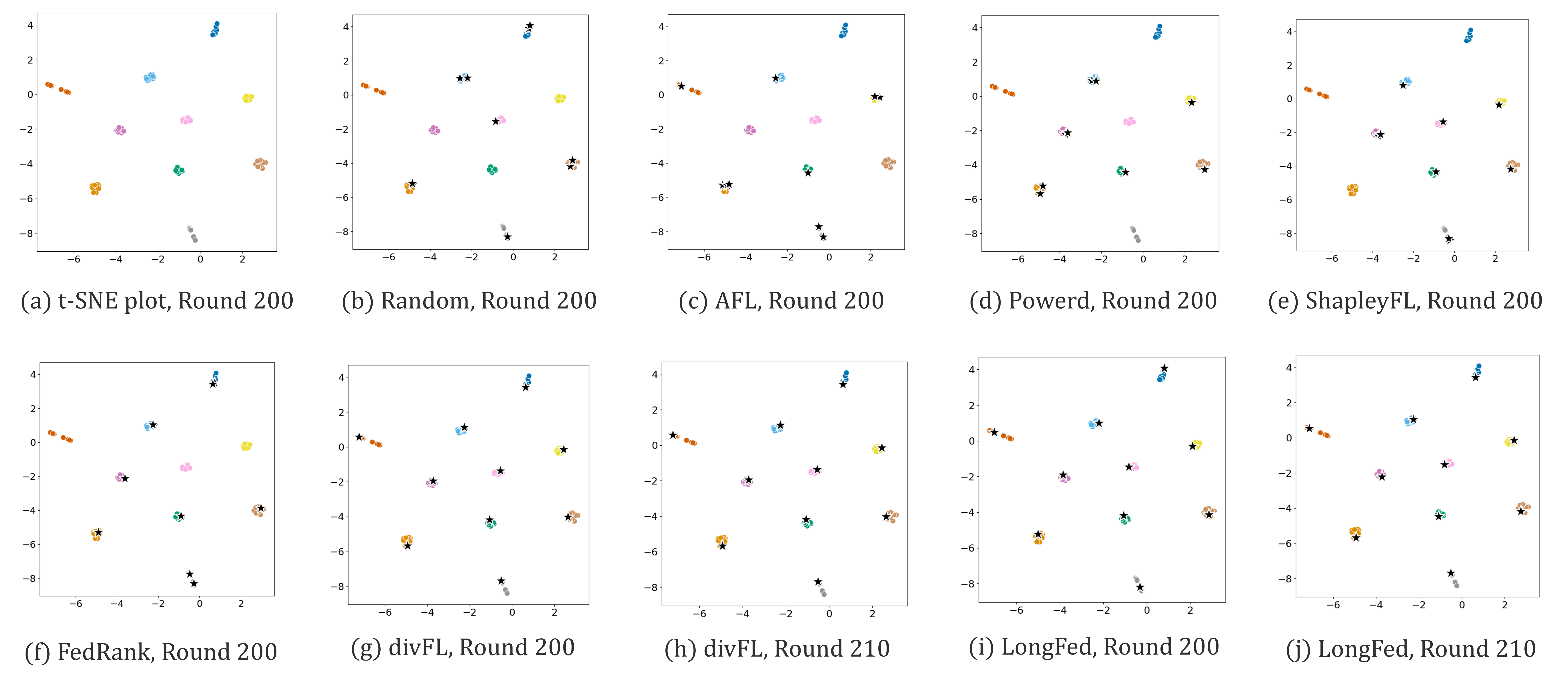}
    \caption{Visualization of the selected clients on the FMNIST dataset under the 1SPC scenario. }
    \label{fig:visual_fmnist_1spc}
\end{figure*}
%(a) is the t-SNE plot \cite{van2008visualizing} of client embeddings,

\begin{figure*}[tbp]
    \centering
    \includegraphics[width=18cm]{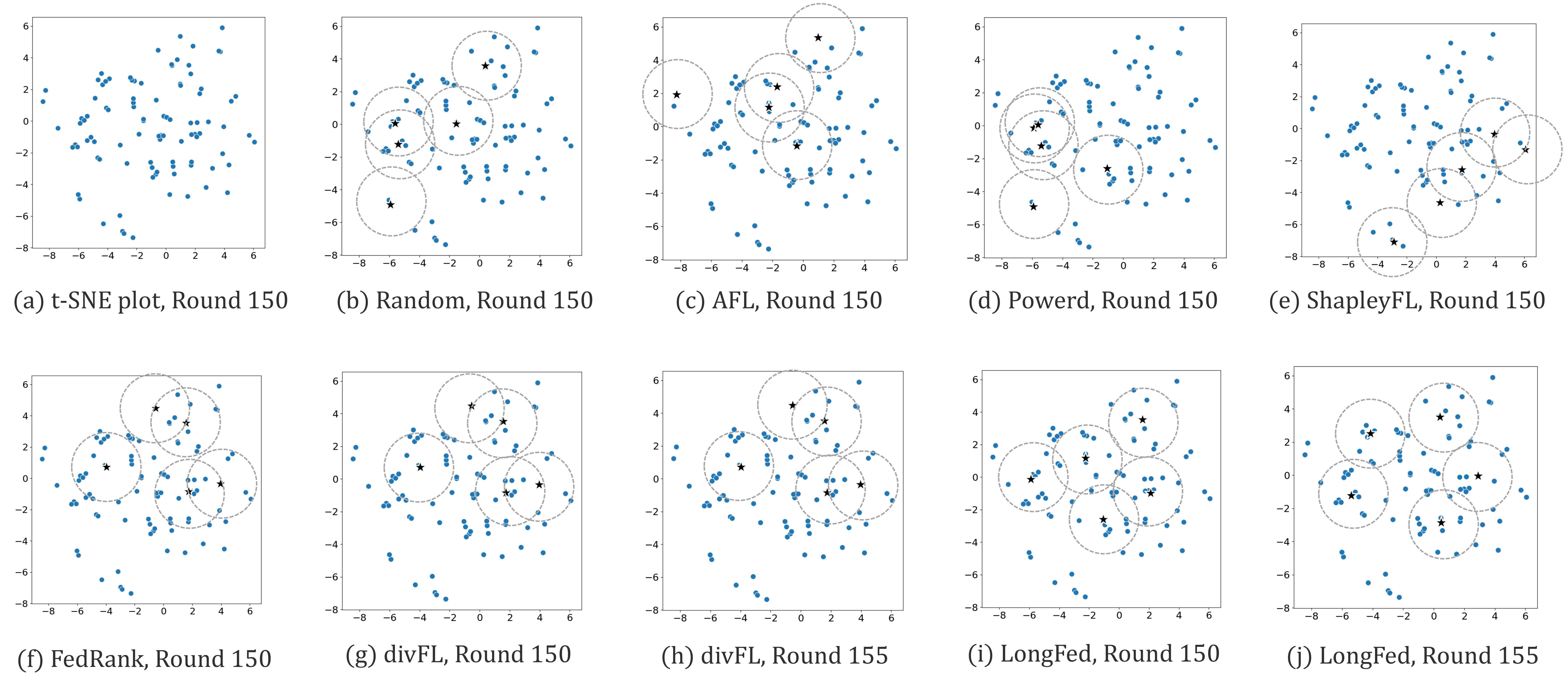}
    \caption{Visualization of the selected clients on the FMNIST dataset under the 2SPC scenario.}
    \label{fig:visual_fmnist_2spc}
\end{figure*}

We provide visualizations of selection results to offer an interpretable analysis. In Fig.~\ref{fig:visual_fmnist_1spc}, (a) is the t-SNE plot of client embeddings, where clients are organized into 10 clusters. (b)-(j) display the results of the selected clients using different client selection strategies, with the chosen clients marked by black stars. As shown in Fig.~\ref{fig:visual_fmnist_1spc} (b)-(f), baseline methods often select two or more clients from the same cluster. In contrast, as shown in Fig.~\ref{fig:visual_fmnist_1spc} (i), the proposed \emph{LongFed} selects one client from each cluster, effectively approximating the data distribution of the full client set. Additionally, as shown in Fig.~\ref{fig:visual_fmnist_1spc} (g) and (h), divFL tends to select the same set of clients across successive rounds. In contrast, as shown in Fig.~\ref{fig:visual_fmnist_1spc} (i) and (j), the proposed \emph{LongFed} selects a more diverse subset of clients across multiple rounds, ensuring fairness in multi-round selection.

Results in the 2SPC scenario are illustrated in Fig.~\ref{fig:visual_fmnist_2spc}. (a) shows the t-SNE plot of client embeddings, where the clustering pattern is not evident. In (b)-(j), circles are drawn around the selected clients, indicating that clients covered by a circle can be represented by the corresponding selected client. For baseline methods, as shown in Fig.~\ref{fig:visual_fmnist_2spc} (b)-(e), the circles around selected clients often overlap and cover only a minority of clients. In our proposed method, as shown in Fig.~\ref{fig:visual_fmnist_2spc} (i), these circles cover the majority of clients. This suggests that the selected clients provide a better approximation of the full client set compared to the baseline approaches. 

both FedRank and divFL repeatedly select the same subset of clients. Notably, this subset remains unchanged after 55 rounds and continues to be selected across nearly 250 rounds. As a result, the global model is trained on a limited set of clients, leading to biased predictions for the remaining clients. This issue is corroborated by the test accuracy results in Fig.\ref{fig:convergence_noniid} (b), where divFL and FedRank exhibit the worst performance. Additionally, as shown in Fig.~\ref{fig:fairness_fmnist}, divFL has the highest standard deviation $\sigma(t)$ among the baselines, further highlighting its fairness limitations. In contrast, \emph{LongFed} maintains a more diverse client selection over multiple rounds, ensuring both fairness and improved generalization. These findings are consistent with the observations in the 1SPC scenario. Additional visualizations for other scenarios are provided in the supplementary file.

Moreover, as shown in Fig.~\ref{fig:visual_fmnist_2spc} (f)-(h), FedRank and divFL select the same subset of clients. Notably, this subset remains unchanged after 55 rounds and continues to be selected across nearly 250 rounds for both FedRank and divFL. As a result, the global model is trained only on the 5 selected clients, leading to biased predictions for the remaining clients. This issue aligns with the test accuracy results in Fig.\ref{fig:convergence_noniid} (b), where divFL and FedRank exhibit the worst performance. Additionally, as shown in Fig.~\ref{fig:fairness_fmnist}, divFL has the highest standard deviation $\sigma(t)$ among the baselines, further highlighting its fairness limitations. In contrast, \emph{LongFed} maintains a more diverse client selection over multiple rounds, ensuring both fairness and improved generalization. These findings are consistent with the observations in the 1SPC scenario. Additional visualizations for other scenarios are provided in the supplementary file. 

\subsection{Individual Fairness Analysis}
\label{sec:parameter_selection}
We analyze the $\epsilon\text{-}\delta$-individual fairness in terms of the trade-off factor $V$, the similarity measure $\epsilon$, and the probability difference measure $\delta$. The results are illustrated using the 2SPC scenario on the CIFAR-10 dataset, and similar trends are observed for other scenarios. We consider two metrics: test accuracy, and the standard deviation $\sigma(t)$ in Eq.~\eqref{equ:select_std}. 

\begin{figure}[tbp]
    \centering
    \includegraphics[width=8.5cm]{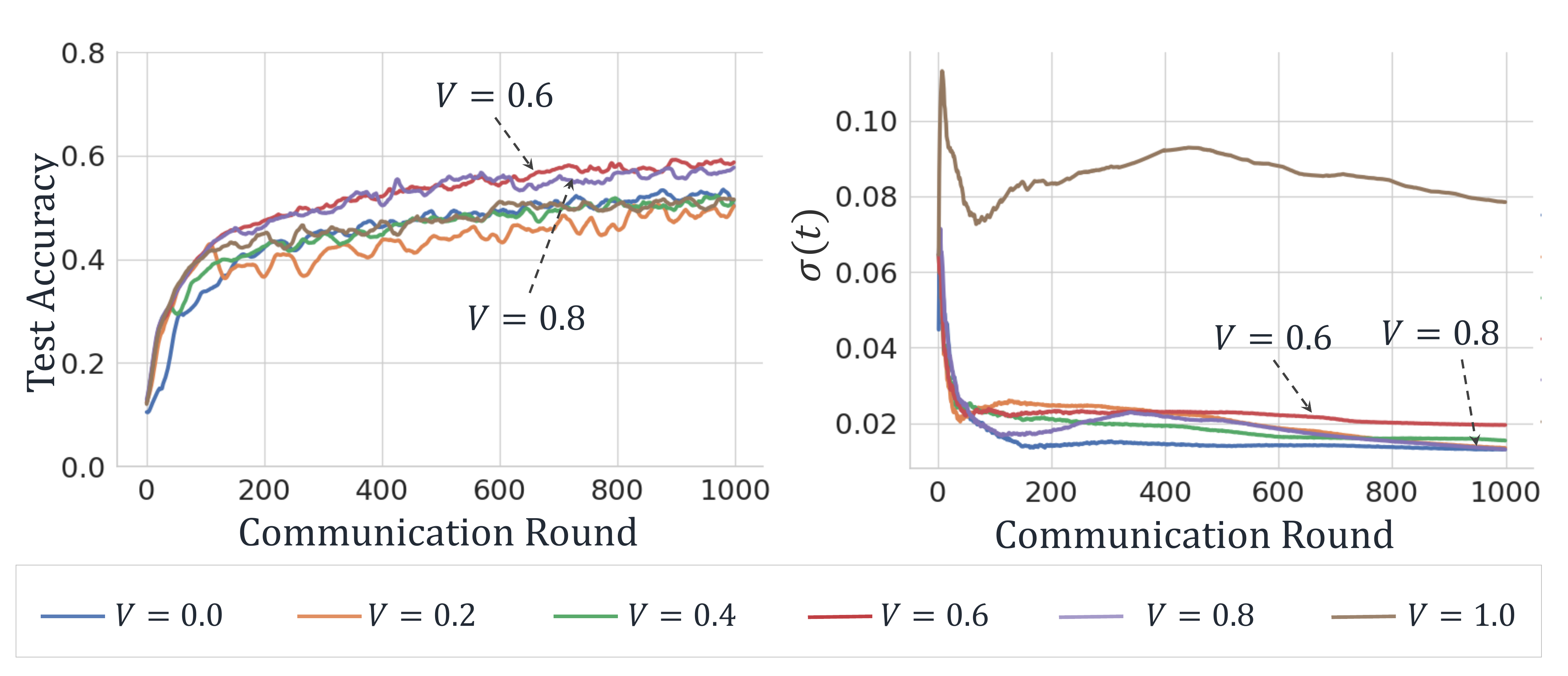}
    \caption{Impact of the trade-off factor $V$.}
    \label{fig:V}
\end{figure}

\textbf{Trade-Off Factor $V$.} We vary the value of $V$ from $0.0$ to $1.0$, where a smaller $V$ indicates a higher priority for the individual fairness constraint. The results are presented in Fig.~\ref{fig:V}. First, we observe that the proposed method achieves the highest test accuracy when $V=0.6$ and $V=0.8$. Second, in terms of fairness, the proposed method exhibits significantly larger standard deviation $\sigma(t)$ when $V=1.0$. This is because when $V=1.0$, the method solely prioritizes minimizing the estimation error in Eq.~\eqref{equ:diversity} while neglecting the fairness constraint. Furthermore, the standard deviation $\sigma(t)$ decreases significantly when $V\leq 0.8$. Considering both model performance and fairness, we recommend setting $V=0.8$ as the optimal choice.

\begin{figure}[tbp]
    \centering
    \includegraphics[width=8.5cm]{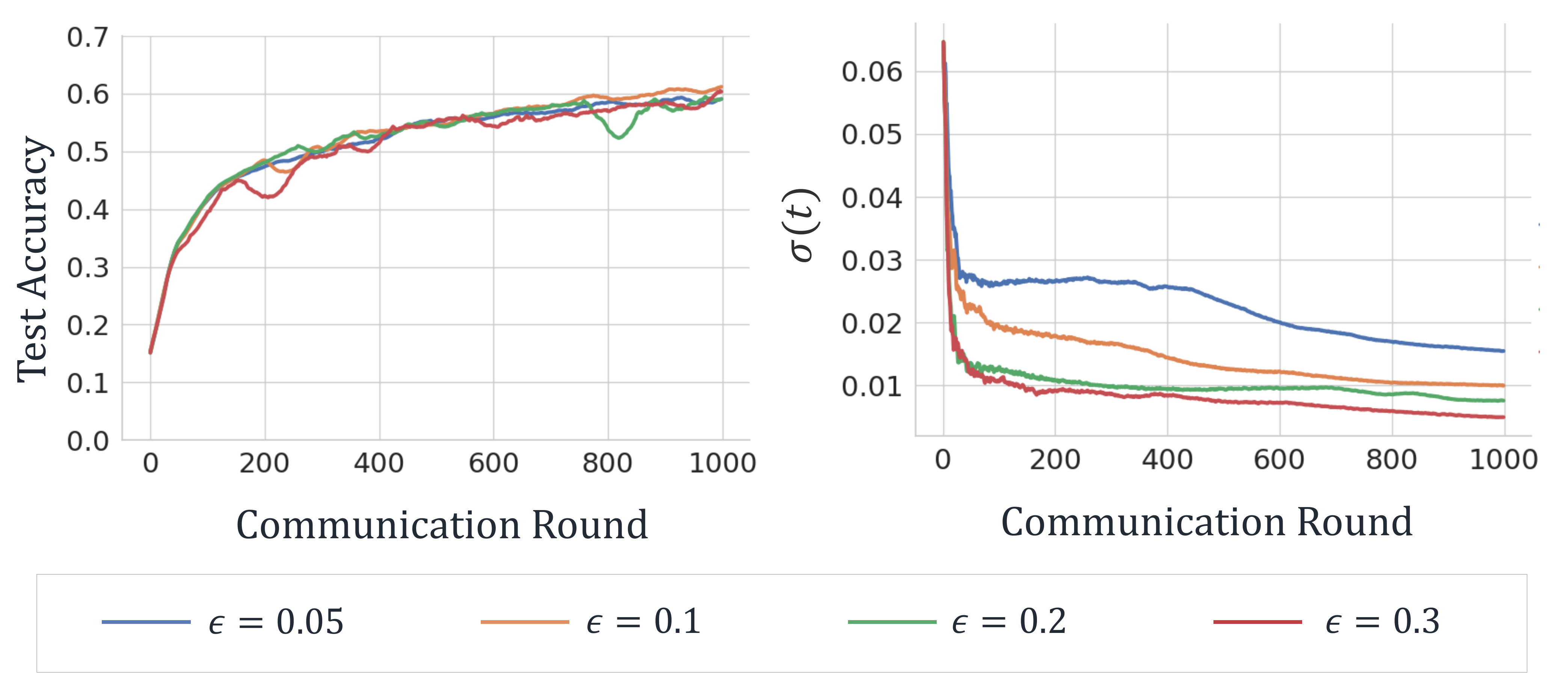}
    \caption{Impact of the similarity measure $\epsilon$.}
    \label{fig:epsilon}
\end{figure}

\textbf{Similarity Measure $\epsilon$.} We vary the value of $\epsilon$ across $0.05$, $0.1$, $0.2$, and $0.3$, and evaluate the corresponding test accuracy and standard deviation, as shown in Fig.~\ref{fig:epsilon}. First, we observe that test accuracy remains relatively consistent and is not particularly sensitive to the choice of $\epsilon$. However, a larger $\epsilon$ results in a smaller standard deviation. This is because $\epsilon$ defines the similarity threshold for grouping clients, and a larger value considers more clients as similar, leading to reduced variability in selection. Considering both model performance and fairness, we recommend setting $\epsilon=0.3$ as the optimal choice.

\begin{figure}[tbp]
    \centering
    \includegraphics[width=8.5cm]{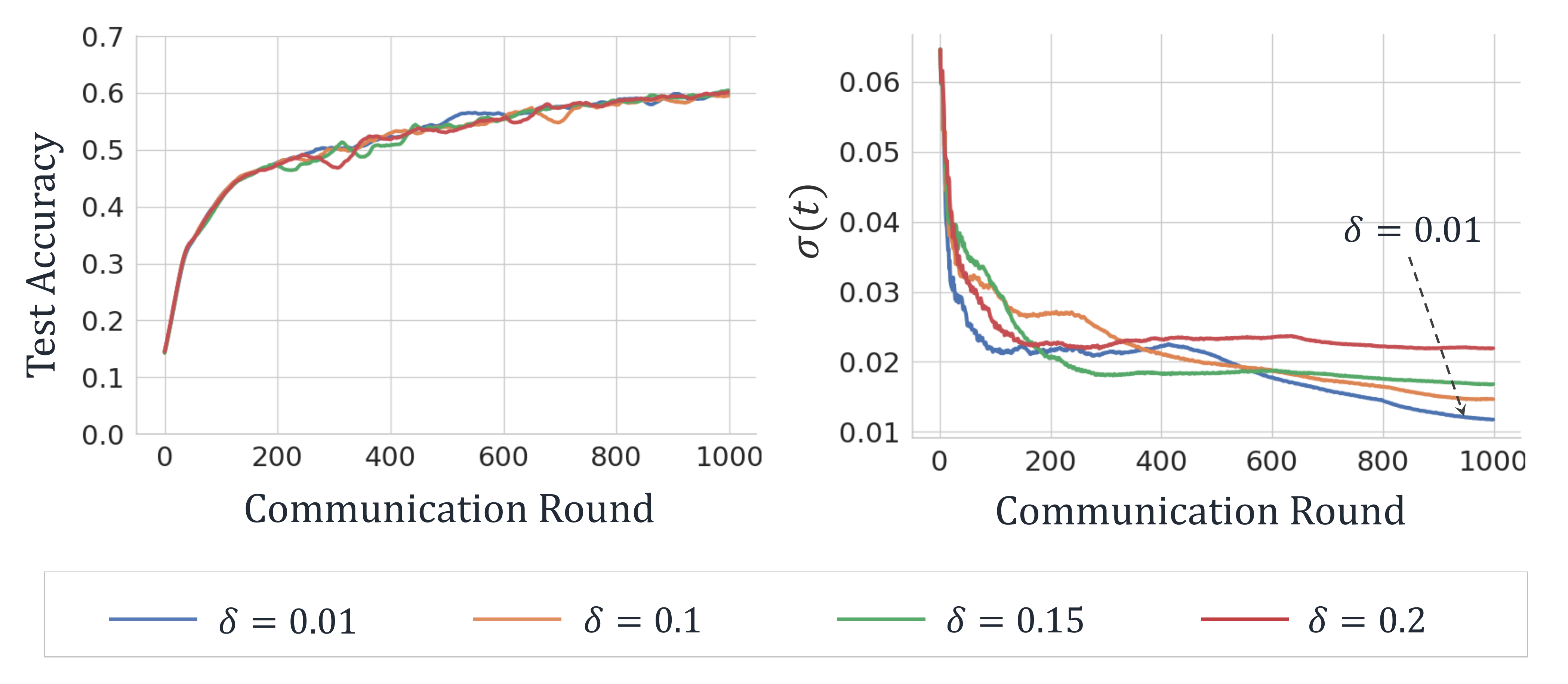}
    \caption{Impact of the probability difference measure $\delta$.}
    \label{fig:delta}
\end{figure}

\textbf{Probability Difference Measure $\delta$.} We vary the value of $\delta$ across $0.01$, $0.1$, $0.15$, and $0.2$, and evaluate the corresponding test accuracy and standard deviation, as shown in Fig.~\ref{fig:delta}. The results indicate that test accuracy remains relatively stable and is not sensitive to the choice of $\delta$. However, a smaller $\delta$ leads to a lower standard deviation. This is because $\delta$ defines the allowable difference in selection probabilities between clients with similar distributions. A smaller $\delta$ enforces a stricter constraint, resulting in a smaller standard deviation. Given the balance between model performance and fairness, we recommend setting $\delta=0.01$ as the best choice.

\subsection{Summary}
In summary, the proposed \emph{LongFed} demonstrates faster convergence and superior test accuracy, effectively enhancing both model performance and fairness simultaneously. Moreover, \emph{LongFed} consistently selects representative clients to approximate full participation across multiple rounds, regardless of whether a clustering pattern exists among clients.

\section{Conclusion}

In this work, we focus on the client selection problem in federated learning, and propose an effective and fair selection method to improve both model performance and fairness. To achieve this, we introduce two guiding principles and formulate the client selection problem as a long-term optimization task. Experiments show that our method effectively guides the system to converge along a trajectory similar to that of full client participation. Visualization results further illustrate that our approach increases data diversity by selecting clients based on their data distributions, thereby improving both model performance and fairness.

\bibliographystyle{IEEEtran}
\bibliography{refs}

\newpage

\onecolumn

\section{Additional Details of The Proposed Optimization Function}
\label{sec:apdx}

\subsection{Proof of Theorem \emph{1}}
\label{apd:proof_diversity_bound}
\begin{proof}
Following the analysis in \cite{mirzasoleiman2020coresets}, based on the mapping $\xi^t:\mathbb{N}\rightarrow \mathbb{S}^t$ that assigns each client $i\in\mathbb{N}$ to a client $j\in\mathbb{S}^t$, we have 
\begin{align}
    \sum_{i\in\mathbb{N}} \nabla f_i(\boldsymbol{w}^t) 
    &\quad = \sum_{i\in\mathbb{N}} \left[\nabla f_i(\boldsymbol{w}^t) - \nabla f_{\xi^t(i)}(\boldsymbol{w}^t) + \nabla f_{\xi^t(i)}(\boldsymbol{w}^t)\right] \\ \notag
    &\quad = \sum_{i\in\mathbb{N}}  \left[\nabla f_i(\boldsymbol{w}^t) - \nabla f_{\xi^t(i)}(\boldsymbol{w}^t) \right] + \sum_{j\in\mathbb{S}^t} \theta_j^t\nabla f_{j}(\boldsymbol{w}^t).  \notag
\end{align}
Subtracting and taking the norm of the both sides, we have  
\begin{align}
    \big\Vert \sum_{i\in\mathbb{N}} \nabla f_i(\boldsymbol{w}^t) - \sum_{j\in\mathbb{S}^t} \theta_j^t \nabla f_{j}(\boldsymbol{w}^t) \big\Vert \leq 
    \sum_{i\in\mathbb{N}}  \big\Vert\nabla f_i(\boldsymbol{w}^t) - \nabla f_{\xi^t(i)}(\boldsymbol{w}^t) \big\Vert.
\end{align}
The upper bound is minimized when $\xi^t$ assigns each client $i\in\mathbb{N}$ to the client in the subset $\mathbb{S}^t$ that has the highest similarity in the gradient space. That is, 
\begin{align}
    \xi^t(i)\in\text{argmin}_{j\in\mathbb{S}^t} \big\Vert\nabla f_{i}(\boldsymbol{w}^t) - \nabla f_{j}(\boldsymbol{w}^t)\big\Vert,
\end{align}
Therefore, we have 
\begin{align}
    \min_{\theta_j^t} \big\Vert \sum_{i\in\mathbb{N}} \nabla f_i(\boldsymbol{w}^t) - \sum_{j\in\mathbb{S}^t} \theta_j^t\nabla f_{j}(\boldsymbol{w}^t) \big\Vert \leq 
    \sum_{i\in\mathbb{N}} \min_{j\in\mathbb{S}^t} \big\Vert\nabla f_i(\boldsymbol{w}^t) - \nabla f_{j}(\boldsymbol{w}^t) \big\Vert,
\end{align}
which completes the proof.
\end{proof}

\subsection{Proof of Theorem \emph{2}}
\label{apd:proof_virtual_queue}
\begin{proof}
We first present the theoretical analysis for $Z_i(t)$. Based on Eq. (13), we have 
\begin{align}
    Z_i(t+1)\geq Z_i(t) + x_{i,t} - x_{i\star,t} - \delta,
\end{align}
which is equivalent to 
\begin{align}
    x_{i,t} - x_{i\star,t} - \delta \leq Z_i(t+1) - Z_i(t).
\end{align}
Accumulating this inequality by $t$ for $t\in[1,T]$, we have
\begin{align}
    \sum_{t=1}^T \left(x_{i,t} - x_{i\star,t} - \delta\right) \leq Z_i(T) - Z_i(0)=Z_i(T).
\end{align}
Taking the expectation operation $\mathbb{E}$ on both sides, we have
\begin{align}
    \frac{1}{T} \sum_{t=1}^T \mathbb{E}\left(x_{i,t} - x_{i\star,t}\right) - \delta \leq \frac{\mathbb{E}[Z_i(T)]}{T}.
\end{align}
It is equivalent to
\begin{align}
    \lim_{T \to +\infty} \frac{\mathbb{E}[Z_i(T)]}{T} = 0 \; \Rightarrow \; \frac{1}{T}\sum_{t=1}^T \mathbb{E} (x_{i,t}-x_{i^\star,t}) - \delta \leq 0, \;
    \forall\, i\in \mathbb{N}.
\end{align}
The proof is similar for $Q_i(t)$, and is omitted here.
\end{proof}

\subsection{Proof of Theorem \emph{3}}
\label{apd:proof_drift_bound}
\begin{proof}
Based on Lemma \ref{lema:drift_bound}, we accumulate the inequality for $Z_i(t)$ in Eq. \eqref{equ:lemma_drift_bound}  by all clients and have
\begin{align}
    \frac{1}{2}\sum_{i=1}^N Z_i^2(t+1)  \leq  \frac{1}{2}\sum_{i=1}^N \left[Z_i(t)+m_i(t)\right]^2 
     = \frac{1}{2}\sum_{i=1}^N Z_i^2(t) + \frac{1}{2}\sum_{i=1}^N m_i^2(t) + \sum_{i=1}^N Z_i(t)m_i(t).
\end{align}
Similarly, for $Q_i(t)$, we have
\begin{align}
    \frac{1}{2}\sum_{i=1}^N Q_i^2(t+1)  \leq  \frac{1}{2}\sum_{i=1}^N \left[Q_i(t)+n_i(t)\right]^2 
     = \frac{1}{2}\sum_{i=1}^N Q_i^2(t) + \frac{1}{2}\sum_{i=1}^N n_i^2(t) + \sum_{i=1}^N Q_i(t)n_i(t).
\end{align}
Then, for the Lyapunov drift $\Delta\left(\Theta(t)\right) $, we have
\begin{align}
    \Delta\left(\Theta(t)\right) 
    &= L\left(\Theta(t+1)\right)- L\left(\Theta(t)\right) \\ \notag
    &= \frac{1}{2} \sum_{i=1}^N \left[Z_i^2(t+1)+Q_i^2(t+1)\right] - \frac{1}{2} \sum_{i=1}^N \left[Z_i^2(t)+Q_i^2(t)\right] \\  \notag 
    &= \frac{1}{2} \sum_{i=1}^N Z_i^2(t+1) - \frac{1}{2} \sum_{i=1}^N Z_i^2(t) + \frac{1}{2} \sum_{i=1}^N Q_i^2(t+1) - \frac{1}{2} \sum_{i=1}^N Q_i^2(t)  \\  \notag
    &\leq \frac{1}{2}\sum_{i=1}^N m_i^2(t) + \sum_{i=1}^N Z_i(t)m_i(t) + \frac{1}{2}\sum_{i=1}^N n_i^2(t) + \sum_{i=1}^N Q_i(t)n_i(t) \\  \notag
    &= \sum_{i=1}^N \left[Z_i(t)m_i(t) + Q_i(t)n_i(t)\right] + \frac{1}{2}\sum_{i=1}^N \left[m_i^2(t)+n_i^2(t)\right],  \\ \notag
    &\leq \sum_{i=1}^N \left[Z_i(t)m_i(t) + Q_i(t)n_i(t)\right] + B,
\end{align}
where $B$ is a positive value that acts as the upper bound for $\frac{1}{2}\sum_{i=1}^N \left[m_i^2(t)+n_i^2(t)\right]$.
\end{proof}

\begin{lemma}
\label{lema:drift_bound}
    Based on Eq. (13), we have 
    \begin{align}
    \label{equ:lemma_drift_bound}
        Z_i^2(t+1)\leq \left[Z_i(t)+m_i(t)\right]^2, \quad \text{and} \quad Q_i^2(t+1)\leq \left[Q_i(t)+n_i(t)\right]^2.
    \end{align}
\end{lemma}

\begin{proof}
First, if $Z_i(t)+m_i(t)\leq 0$, then $Z_i(t+1)=Z_i(t)+m_i(t)$, and we have 
\begin{align}
    Z_i^2(t+1) = \left[Z_i(t)+m_i(t)\right]^2.
\end{align}
Next, if $Z_i(t)+m_i(t)< 0$, then $Z_i(t+1)=0>Z_i(t)+m_i(t)$, and we have 
\begin{align}
    Z_i^2(t+1) < \left[Z_i(t)+m_i(t)\right]^2
\end{align}
Combining the two cases, we have
\begin{align}
    Z_i^2(t+1) \leq \left[Z_i(t)+m_i(t)\right]^2
\end{align}
The analysis is similar for $Q_i(t)$, and is omitted here. 
\end{proof}

\section{Additional Details of The Convergence Analysis}
\label{apd:convergence}

\subsection{Proof of Theorem \emph{4}}

\begin{proof}
Our theoretical analysis is based on the FedAvg \cite{mcmahan2017communication} method, and it can also be extended to other federated optimization methods. To align with the approach in \cite{li2020convergence}, we unify the epochs of local training in clients and communication rounds for parameter transmission between the server and clients into a single dimension, indexed by $t = sE+k$. Here, $s$ denotes the index of the current communication round, $E$ is the number of local epochs in a communication round, and $k\in [1,E-1]$. If $t$ is divisible by $E$ (indicated as $t\mid E$), it signifies the communication step where the server aggregates the model parameters from the selected clients. Otherwise, it represents a local training step for clients.

To show the convergence, we introduce an auxiliary variable $\boldsymbol{v}_i^{t}$ to signify the immediate result of a single stochastic gradient descent (SGD) step in local updates. That is, 
\begin{align}
    \boldsymbol{v}_i^{t+1} = \boldsymbol{w}_i^t - \eta_t \nabla f_i(\boldsymbol{w}_i^t, \beta_i^t), \quad \text{and} \quad 
    \boldsymbol{w}_i^{t} = 
    \begin{cases}
        \sum_{i\in \mathbb{S}^{t}}\theta_i^t\boldsymbol{v}_i^{t}, \; &\text{if} \; t\mid E, \\
        \boldsymbol{v}_i^{t}, \; &\text{otherwise}. 
    \end{cases}
\end{align}
Based on $\boldsymbol{v}_i^{t}$ and $\boldsymbol{w}_i^t$, we define two virtual sequences $\bar{\boldsymbol{v}}^{t}$ and $\bar{\boldsymbol{w}}^t$, 
\begin{align}
    \bar{\boldsymbol{v}}^{t} = \sum_{i} \theta_i^t \boldsymbol{v}_i^t, \quad \text{and} \quad
    \bar{\boldsymbol{w}}^t = \sum_{i} \theta_i^t \boldsymbol{w}_i^t.
\end{align}
and define
\begin{align}
    \bar{\boldsymbol{g}}^t = \sum_{i=1}^N \theta_i^t \nabla f_i(\boldsymbol{w}_i^t) \quad \text{and} \quad  \boldsymbol{g}^t = \sum_{i=1}^N \theta_i^t \nabla f_i(\boldsymbol{w}_i^t, \alpha_i^t).
\end{align}
Therefore, we have 
\begin{align}
     \bar{\boldsymbol{v}}^{t+1} = \bar{\boldsymbol{w}}^t - \eta_t \boldsymbol{g}^t \quad \text{and} \quad \mathbb{E}(\boldsymbol{g}^t) = \bar{\boldsymbol{g}}^t
\end{align}
Note that
\begin{align}
    \Vert \bar{\boldsymbol{w}}^{t+1}-\boldsymbol{w}^\star \Vert^2 &= \Vert \bar{\boldsymbol{w}}^{t+1}-\bar{\boldsymbol{v}}^{t+1} + \bar{\boldsymbol{v}}^{t+1}- \boldsymbol{w}^\star \Vert^2 \\ \notag
    &= \underbrace{\Vert \bar{\boldsymbol{w}}^{t+1}-\bar{\boldsymbol{v}}^{t+1} \Vert^2}_{A_1} 
    + \underbrace{\Vert \bar{\boldsymbol{v}}^{t+1}- \boldsymbol{w}^\star \Vert^2}_{A_2} 
    + \underbrace{2\cdot\langle\bar{\boldsymbol{w}}^{t+1}-\bar{\boldsymbol{v}}^{t+1},\; \bar{\boldsymbol{v}}^{t+1}- \boldsymbol{w}^\star\rangle}_{A_3}. 
\end{align}
That is, we can bound $\Vert \bar{\boldsymbol{w}}^{t+1}-\boldsymbol{w}^\star \Vert$ by obtaining the upper bounds of the three terms, i.e., $A_1$, $A_2$, and $A_3$, respectively.

\textbf{Upper bound of Term $A_1$.} Consider the last time of aggregation occurs at the step $t_0=t+1-E$, and  let $\Delta \boldsymbol{v}^\tau_i=\boldsymbol{v}^{\tau+1}_i-\boldsymbol{v}^\tau_i$ be the updates on $\boldsymbol{v}^\tau_i$ at the $\tau$-th step, then we have  
\begin{align}
    \bar{\boldsymbol{v}}^{t+1} = \bar{\boldsymbol{w}}^{t_0} + \frac{1}{N}\sum_{i=1}^N \sum_{\tau=t_0}^t \Delta \boldsymbol{v}^\tau_i.
\end{align}
The term $A_1$ is equivalent to
\begin{align}
    \Vert \bar{\boldsymbol{w}}^{t+1}-\bar{\boldsymbol{v}}^{t+1} \Vert^2 
    &= \left\Vert \left(\bar{\boldsymbol{w}}^{t_0}+\frac{1}{N}\sum_{i\in \mathbb{S}^t} \theta_i^t \sum_{\tau=t_0}^t\Delta \boldsymbol{v}^\tau_i\right) - 
    \left(\bar{\boldsymbol{w}}^{t_0}+ \frac{1}{N}\sum_{i=1}^N \sum_{\tau=t_0}^t \Delta \boldsymbol{v}^\tau_i\right)\right\Vert  \\ \notag
    &= \left\Vert \sum_{\tau=t_0}^t \left(\frac{1}{N} \sum_{i\in \mathbb{S}^t} \theta_i^t \Delta \boldsymbol{v}^\tau_i -\frac{1}{N}\sum_{i=1}^N \Delta \boldsymbol{v}^\tau_i\right)\right\Vert \\ \notag
    &\leq \sum_{\tau=t_0}^t \left\Vert \frac{1}{N} \sum_{i\in \mathbb{S}^t} \theta_i^t \Delta \boldsymbol{v}^\tau_i -\frac{1}{N}\sum_{i=1}^N \Delta \boldsymbol{v}^\tau_i \right\Vert.
\end{align}
Note that for every local step $\tau\in (t_0,t]$, we use the same $\mathbb{S}^t$ to approximate the full gradients. Based on Assumption 6, we have 
\begin{align}
    \left\Vert \frac{1}{N} \sum_{i\in \mathbb{S}^t} \theta_i^t \nabla f_i (\boldsymbol{v}^{\tau}_i)  -\frac{1}{N}\sum_{i=1}^N \nabla f_i (\boldsymbol{v}^{\tau}_i) \right\Vert
    & \leq \left\Vert \frac{1}{N} \sum_{i\in \mathbb{S}^t} \theta_i^t \nabla f_i (\boldsymbol{v}^{\tau}_i)  -\frac{1}{N}\sum_{i\in \mathbb{S}^t} \theta_i^t \nabla f_i (\boldsymbol{v}^{t_0}_i) \right\Vert \\
    & \qquad + \left\Vert \frac{1}{N} \sum_{i\in \mathbb{S}^t} \theta_i^t \nabla f_i (\boldsymbol{v}^{\tau}_i)  -\frac{1}{N}\sum_{i=1}^N \nabla f_i (\boldsymbol{v}^{\tau}_i) \right\Vert \\
    & \qquad + \left\Vert \frac{1}{N}\sum_{i=1}^N \nabla f_i (\boldsymbol{v}^{t_0}_i)  -\frac{1}{N}\sum_{i=1}^N \nabla f_i (\boldsymbol{v}^{\tau}_i) \right\Vert \\
    & \leq 2LB_3\sum_{v=t_0}^\tau \eta_v + \rho
\end{align}
where the first and third terms on the right-hand side are bounded by Assumption 1 and Assumption 3, respectively. Therefore, the term $A_1$ is bounded by 
\begin{align}
    \Vert \bar{\boldsymbol{w}}^{t+1}-\bar{\boldsymbol{v}}^{t+1} \Vert^2 
    & \leq \sum_{\tau=t_0}^t \left\Vert \frac{1}{N} \sum_{i\in \mathbb{S}^t} \theta_i^t \Delta \boldsymbol{v}^\tau_i -\frac{1}{N}\sum_{i=1}^N \Delta \boldsymbol{v}^\tau_i \right\Vert \\ \notag
    & = \sum_{\tau=t_0}^t \eta_\tau \left\Vert \frac{1}{N} \sum_{i\in \mathbb{S}^t} \theta_i^t \nabla f_i (\boldsymbol{v}^{\tau}_i) -\frac{1}{N}\sum_{i=1}^N \nabla f_i (\boldsymbol{v}^{\tau}_i) \right\Vert \\ \notag
    & \leq 2LB_3 \sum_{\tau=t_0}^t \sum_{v=t_0}^\tau \eta_\tau \eta_v + E\rho\eta_\tau \\ \notag
    &\leq LB_3E(E-1)\eta_{t_0}^2 + E\rho\eta_{t_0} \\ \notag
    % & = LGE(E-1)\left(1+\frac{E-1}{t+\gamma-(E-1)}\right)^2\eta_t^2 + E\left(1+\frac{E-1}{t+\gamma-(E-1)}\right)\eta_t\epsilon
\end{align}

\textbf{Upper bound of $A_2$.} Under the Assumptions 1 and 2, based on the Lemma 1 in \cite{li2020convergence}, we have 
\begin{align}
    \mathbb{E}\Vert \bar{\boldsymbol{v}}^{t+1}- \boldsymbol{w}^\star \Vert^2 \leq 
    (1-\eta_t\mu)\mathbb{E}\Vert \bar{\boldsymbol{w}}^{t}- \boldsymbol{w}^\star \Vert^2 + \eta_t^2 C_1,
\end{align}
where $C_1$ is a constant. 

\textbf{Upper bound of $A_3$. } Following the proof in \cite{balakrishnan2022diverse}, we have  $\mathbb{E}\left[\Vert \bar{\boldsymbol{v}}^{t+1}- \boldsymbol{w}^\star \Vert\right]$ can be bounded by a constant $C_2$, which is determined by the value of ${B_3}/{\mu}$.
% \begin{align}
%     \mathbb{E}\left[\Vert \bar{\boldsymbol{v}}^{t+1}- \boldsymbol{w}^\star \Vert\right] 
    % &\leq \mathbb{E}\left\Vert \bar{\boldsymbol{v}}^{t+1} - \sum_{i=1}^N \theta_i^t \boldsymbol{v}_i^\star \right\Vert 
    % + \mathbb{E}\left\Vert \sum_{i=1}^N \theta_i^t \boldsymbol{v}_i^\star - \boldsymbol{w}^\star\right\Vert \\ \notag
    % &\leq \mathbb{E}\left\Vert \bar{\boldsymbol{v}}^{t+1} - \sum_{i=1}^N \theta_i^t \boldsymbol{v}_i^\star \right\Vert \\ \notag
    % &\leq \sum_{i=1}^N \mathbb{E} \left\Vert \theta_i^t \left(\bar{\boldsymbol{v}}^{t+1} - \boldsymbol{v}_i^\star\right)\right\Vert \\ \notag
    % &\leq \sum_{i=1}^N \frac{\theta_i^t}{\mu} \mathbb{E} \left\Vert \nabla f_i (\boldsymbol{v}^{t}_i)\right\Vert \leq \frac{B_3}{\mu}.
% \end{align}

Based on the above analysis, we have
\begin{align}
    \mathbb{E}  \Vert \bar{\boldsymbol{w}}^{t+1}-\boldsymbol{w}^\star \Vert^2 
    & \leq \mathbb{E}\Vert \bar{\boldsymbol{w}}^{t+1}-\bar{\boldsymbol{v}}^{t+1} \Vert^2 
    + \mathbb{E} \Vert \bar{\boldsymbol{v}}^{t+1}- \boldsymbol{w}^\star \Vert^2 
    + 2\cdot\mathbb{E} \left[\langle\bar{\boldsymbol{w}}^{t+1}-\bar{\boldsymbol{v}}^{t+1},\; \bar{\boldsymbol{v}}^{t+1}- \boldsymbol{w}^\star\rangle\right]   \\ \notag
    & \leq \left[LB_3E(E-1)\eta_{t_0}^2 + E\rho\eta_{t_0}\right]^2 + \left[(1-\eta_t\mu)\mathbb{E}\Vert \bar{\boldsymbol{w}}^{t}- \boldsymbol{w}^\star \Vert^2 + \eta_t^2 C_1\right] \\ \notag
    & \qquad \qquad \qquad \qquad \qquad \qquad + 2\left[LB_3E(E-1)\eta_{t_0}^2 + E\rho\eta_{t_0} \right] 
    \cdot \mathbb{E} \Vert \bar{\boldsymbol{v}}^{t+1}- \boldsymbol{w}^\star \Vert^2 \\ \notag
    & \leq (1-\eta_t\mu)\mathbb{E}\Vert \bar{\boldsymbol{w}}^{t}- \boldsymbol{w}^\star \Vert^2
    + \left[LB_3E(E-1)C_2+\left(LB_3E(E-1)\eta_{t_0}+E\rho\right)^2\right]\eta_{t_0}^2 \\ \notag
    &\qquad \qquad \qquad \qquad \qquad \qquad + EC_2\rho\eta_{t_0} + C_1\eta_{t}^2 \\ \notag
    &\leq (1-\eta_t\mu)\mathbb{E}\Vert \bar{\boldsymbol{w}}^{t}- \boldsymbol{w}^\star \Vert^2
    + \mathcal{O}(\rho) + \mathcal{O}(\eta_t^2) + \mathcal{O}(\eta_{t_0}^4).
\end{align}
By letting 
\begin{align}
    \eta_t=\frac{\beta}{t+\gamma}, \quad \text{and} \quad \eta_{t_0}=\frac{\beta}{t+1-E+\gamma}
\end{align}
with $\beta>{1}/{\mu}$ and $\gamma>0$ to achieve a diminishing learning rate, we complete the proof.

\end{proof}

\section{Additional Details of Experiment}
\subsection{Experimental Settings}
Following \cite{tang2022fedcor}, for the FMNIST dataset, we employ a multilayer perceptron with two hidden layers as the global model, where the number of units in the two hidden layers is 64 and 30, respectively. Additionally, we set the number of local epochs to 3, the local batch size to 64, and the learning rate to $0.005$. For the CIFAR-10 dataset, the architecture of the global model is a convolutional neural network (CNN) with three convolutional layers having 32, 64, and 64 kernels, respectively. All kernels are designed with the size $3 \times 3$. Besides, the outputs of the convolutional layers are fed into an MLP layer with 64 units. We set the number of local epochs to 5, the local batch size to 128, and the learning rate to 0.05.

\subsection{Visualization of Client Selection Strategy}
The visualization results for the FMNIST dataset under the Dir scenario are presented in Fig.~\ref{fig:visual_fmnist_dir}. Further, we provide visualizations for the CIFAR-10 dataset in the 1SPC, 2SPC, and Dir scenarios in Fig. \ref{fig:visual_cifar_1spc}, Fig. \ref{fig:visual_cifar_2spc}, and Fig. \ref{fig:visual_cifar_dir}, respectively. The conclusions drawn from the CIFAR-10 dataset align with those from the FMNIST dataset.

\begin{figure}[tbp]
    \centering
    \includegraphics[width=17cm]{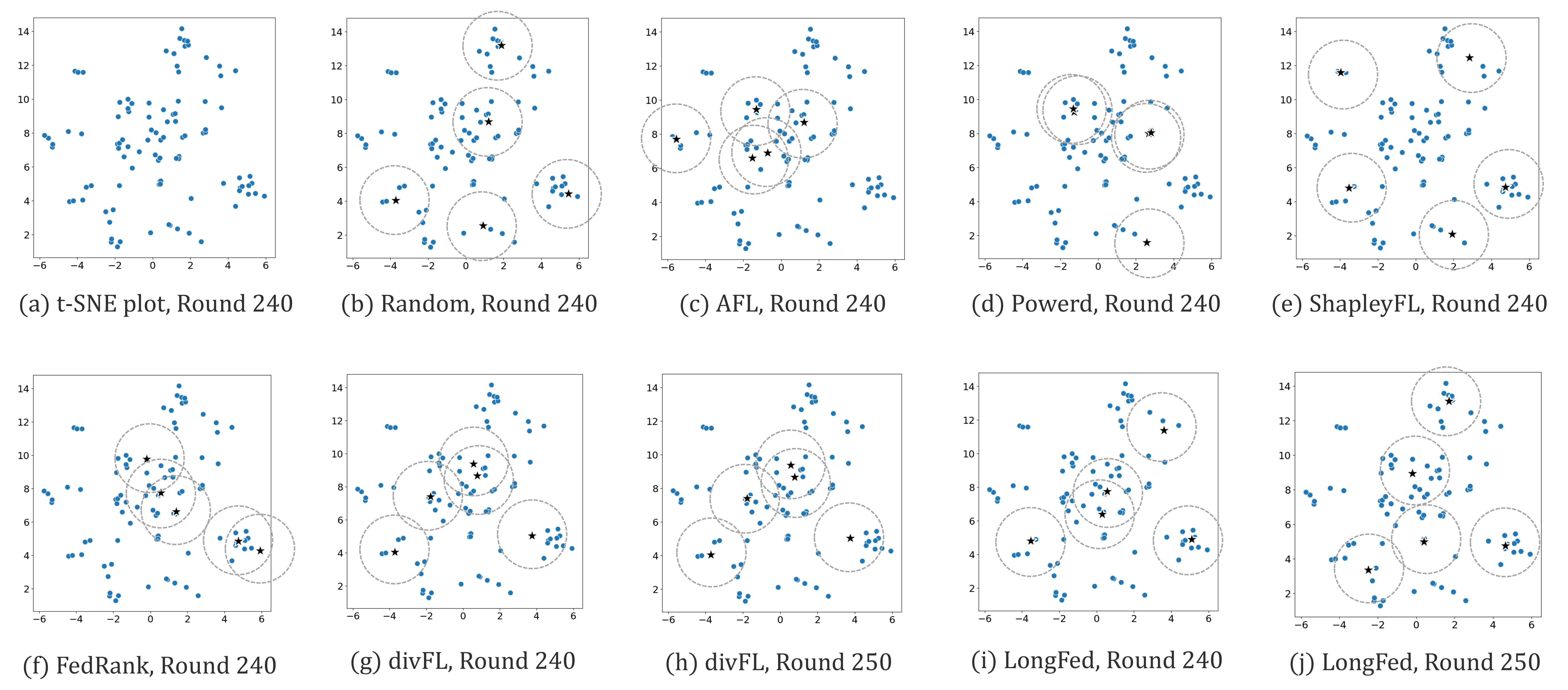}
    \caption{Visualization of selected clients on FMNIST under the Dir scenario.}
    \label{fig:visual_fmnist_dir}
\end{figure}

\begin{figure}[tbp]
    \centering
    \includegraphics[width=17cm]{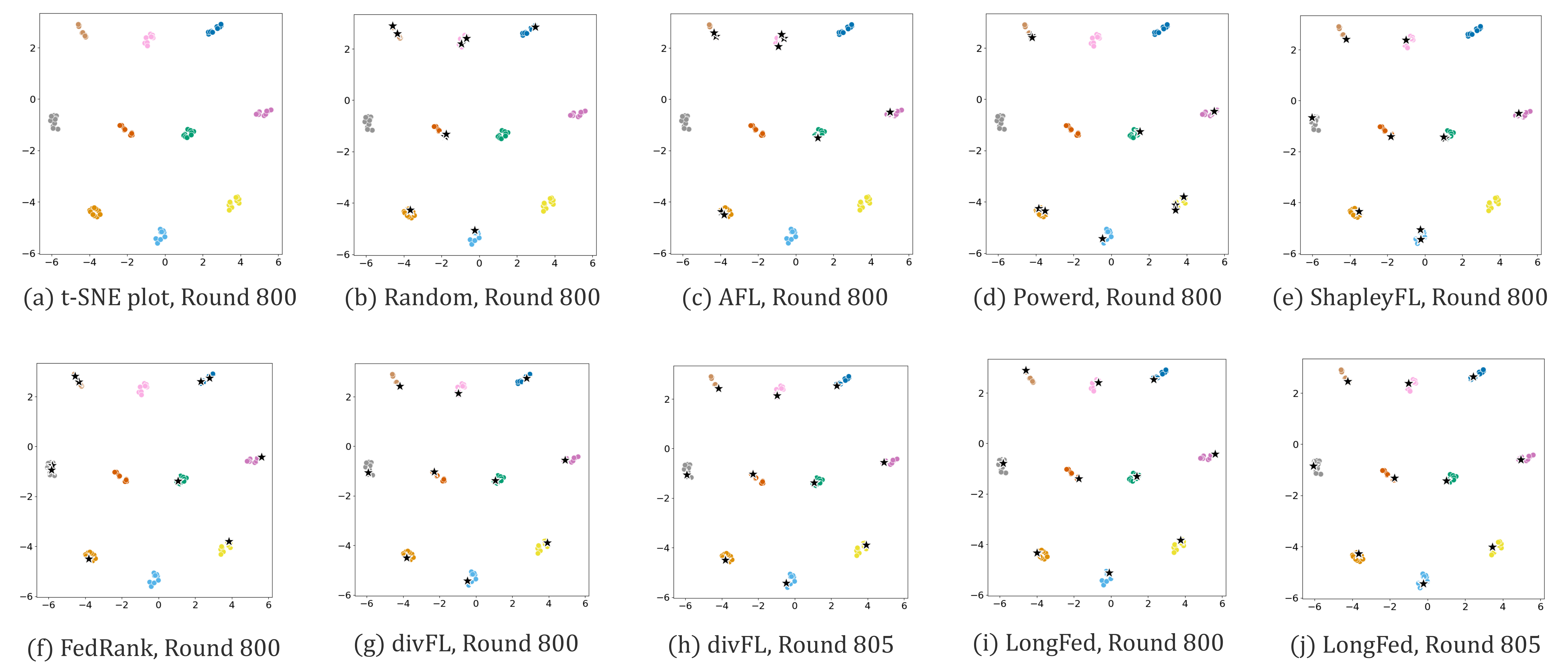}
    \caption{Visualization of selected clients on CIFAR-10 under the 1SPC scenario.}
    \label{fig:visual_cifar_1spc}
\end{figure}

\begin{figure}[tbp]
    \centering
    \includegraphics[width=17cm]{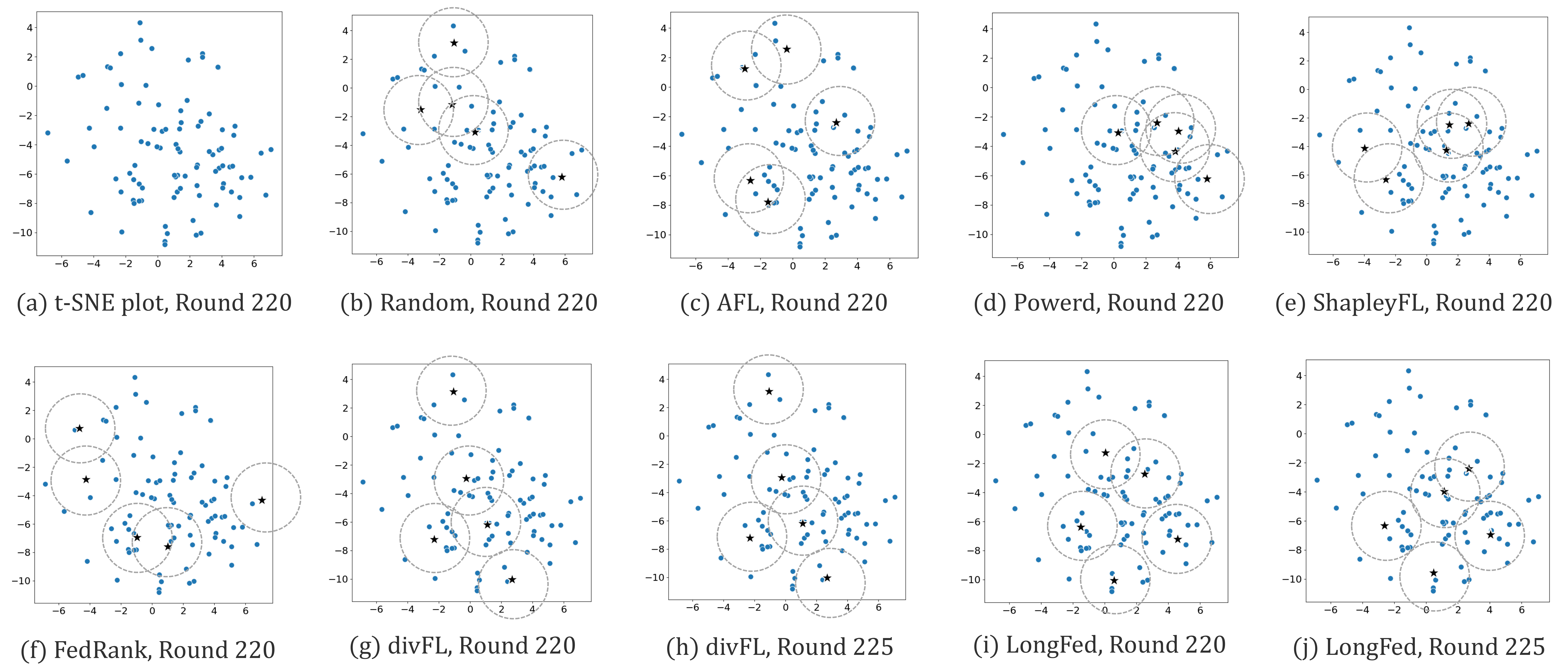}
    \caption{Visualization of selected clients on CIFAR-10 under the 2SPC scenario.}
    \label{fig:visual_cifar_2spc}
\end{figure}

\begin{figure}[tbp]
    \centering
    \includegraphics[width=17cm]{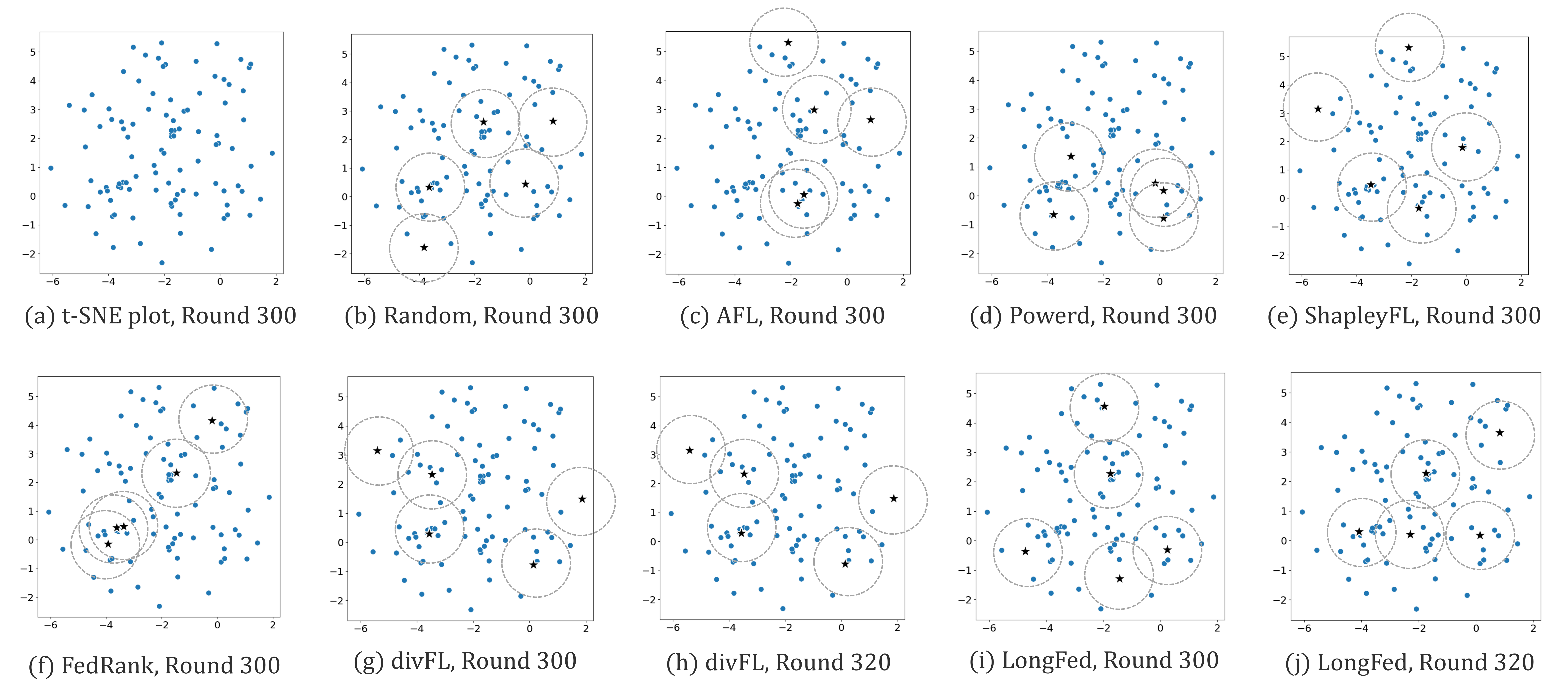}
    \caption{Visualization of selected clients on CIFAR-10 under the Dir scenario.}
    \label{fig:visual_cifar_dir}
\end{figure}


% \end{IEEEbiography}

\end{document}